\documentclass[10pt, a4paper]{article} 
\usepackage[square,numbers,sort&compress]{natbib}
\usepackage[shortlabels]{enumitem}
\usepackage{setspace}
\usepackage{subfloat}
\usepackage{mathtools}
\usepackage{amsmath,amsfonts,amssymb,amsthm,mathrsfs, bm}
\usepackage[utf8]{inputenc}
\usepackage[swedish, english]{babel}
\usepackage{csquotes}
\usepackage{subfig}
\usepackage{graphicx}
\usepackage{supertabular}                                             
\usepackage[table]{xcolor} 

\usepackage{graphicx}
\usepackage{epstopdf}
\usepackage{microtype}
\usepackage{multirow}
\usepackage{hyperref}
\usepackage{xcolor}
\usepackage{authblk}
\usepackage{datetime2}	
\usepackage[left=1in, right=1in, top=1in, bottom=1in]{geometry}
\usepackage{setspace}	
\theoremstyle{plain}
\newtheorem{theorem}{Theorem}[section]
\newtheorem{lemma}{Lemma}[section]

\newtheorem{proposition}{Proposition}[section]

\usepackage{amsmath,amsfonts,amssymb,amsthm,mathrsfs, bm}
\usepackage{stmaryrd}

\usepackage{graphicx}
\usepackage{tikz}
\usepackage[utf8]{inputenc}
\usepackage{pgfplots} 
\usepackage{pgfgantt}
\usepackage{pdflscape}
\pgfplotsset{compat=newest} 
\pgfplotsset{plot coordinates/math parser=false}

\usepackage{endnotes}
\usepackage{commath}
\usepackage{gensymb}
\usepackage{mathtools}
\usepackage{tikz} 
\usetikzlibrary{calc,quotes,angles}
\usepackage{tkz-euclide}
\usetikzlibrary{automata,positioning} 

\usepackage{xcolor}
\usepackage{tikz,tkz-euclide,siunitx}
\usepackage{etoolbox} 

\usetikzlibrary{quotes,arrows.meta,angles,calc,shadings,positioning,babel}

\newtheorem{remark}{Remark}[section]

\usepackage{endnotes}
\usepackage{commath}
\usepackage{gensymb}
\usepackage{textgreek} 

\makeatletter
\patchcmd{\tkz@DrawLine}{\begingroup}{\begingroup\makeatletter}{}{}
\makeatother

\allowdisplaybreaks
\DeclareMathOperator{\RE}{Re}

\DeclareMathOperator{\sgn}{sgn}

\DeclareMathOperator{\esup}{ess\, sup}

\DeclareMathOperator{\diag}{\mathrm{diag}}
\makeatletter
\newcommand\makebig[2]{%
  \@xp\newcommand\@xp*\csname#1\endcsname{\bBigg@{#2}}%
  \@xp\newcommand\@xp*\csname#1l\endcsname{\@xp\mathopen\csname#1\endcsname}%
  \@xp\newcommand\@xp*\csname#1r\endcsname{\@xp\mathclose\csname#1\endcsname}%
}
\makeatother

\providecommand*{\ped}[1]{%
\ensuremath{_\textnormal{#1}}}

\providecommand*{\eu}%
{\ensuremath{\mathrm{e}}}
\providecommand*{\im}%
{\ensuremath{\mathrm{i}}}
\providecommand*{\GammaF}%
{\ensuremath{\mathrm{\Gamma}}}
\providecommand*{\BetaF}%
{\ensuremath{\mathrm{\Beta}}}

\makebig{biggg} {3.0}
\makebig{Biggg} {3.5}
\makebig{bigggg}{4.0}
\makebig{Bigggg}{4.5}
\makebig{biggggg}{5.0}
\makebig{Biggggg}{5.5}
\usepackage{longtable}
\usepackage{xcolor}
\DeclareMathSymbol{\Gamma}{\mathalpha}{letters}{"00}
\DeclareMathSymbol{\Delta}{\mathalpha}{letters}{"01}
\DeclareMathSymbol{\Theta}{\mathalpha}{letters}{"02}
\DeclareMathSymbol{\Lambda}{\mathalpha}{letters}{"03}
\DeclareMathSymbol{\Xi}{\mathalpha}{letters}{"04}
\DeclareMathSymbol{\Pi}{\mathalpha}{letters}{"05}
\DeclareMathSymbol{\Sigma}{\mathalpha}{letters}{"06}
\DeclareMathSymbol{\Upsilon}{\mathalpha}{letters}{"07}
\DeclareMathSymbol{\Phi}{\mathalpha}{letters}{"08}
\DeclareMathSymbol{\Psi}{\mathalpha}{letters}{"09}
\DeclareMathSymbol{\Omega}{\mathalpha}{letters}{"0A}

\usepackage{hyperref}
\hypersetup{
    colorlinks = true,
    linkbordercolor = {white},
            linkcolor=blue,
            bookmarks=true,
            filecolor=blue,
            urlcolor=blue,
            citecolor=blue
}

\begin{document}

\title{Semilinear single-track vehicle models\\ with distributed tyre friction dynamics}
\date{}
\author[a,b]{Luigi Romano\thanks{Corresponding author. Email: luigi.romano@liu.se.}}
\author[b]{Ole Morten Aamo}
\author[a]{Jan Aslund}
\author[a]{Erik Frisk}
\affil[a]{\footnotesize{Division of Vehicular Systems, Department of Electrical Engineering, Linköping University, SE-581 83 Linköping, Sweden}}
\affil[b]{\footnotesize{Department of Engineering Cybernetics, Norwegian University of Science and Technology, O. S. Bragstads plass 2, NO-7034, Trondheim, Norway}}

\maketitle

\begin{abstract}
This paper introduces a novel family of single-track vehicle models that incorporate a distributed representation of transient tyre dynamics, whilst simultaneously accounting for nonlinear effects induced by friction. The core of the proposed framework is represented by the distributed \emph{Friction with Bristle Dynamics} (FrBD) model, which unifies and extends classical formulations such as Dahl and LuGre by describing the rolling contact process as a spatially distributed system governed by semilinear \emph{partial differential equations} (PDEs). This model is systematically integrated into a single-track vehicle framework, where the resulting semilinear ODE-PDE interconnection captures the interaction between lateral vehicle motion and tyre deformation. Two main variants are considered: one with rigid tyre carcass and another with flexible carcass, each admitting a compact state-space representation. Local and global well-posedness properties for the coupled system are established rigorously, highlighting the dissipative and physically consistent properties of the distributed FrBD model. A linearisation procedure is also presented, enabling spectral analysis and transfer function derivation, and potentially facilitating the synthesis of controllers and observers. Numerical simulations demonstrate the model's capability to capture micro-shimmy oscillations and transient lateral responses to advanced steering manoeuvres. The proposed formulation advances the state-of-the-art in vehicle dynamics modelling by providing a physically grounded, mathematically rigorous, and computationally tractable approach to incorporating transient tyre behaviour in lateral vehicle dynamics, when accounting for the effect of limited friction.
\end{abstract}
\section*{Keywords}
Single-track model; tyre models; transient tyre dynamics; distributed friction; distributed parameter systems; semilinear systems

\section{Introduction}\label{intro}
In modern vehicle dynamics, accurate modelling of tyre-road interaction remains a cornerstone for understanding complex dynamic phenomena \cite{USB,LibroMio,Meccanica2,LuGreSpin,CarcassDyn,MioSolo} and developing advanced control strategies aimed at enhancing vehicle safety and performance \cite{Rajamani,Nielsen,Savaresi,LuGreControl2,deCastro2,Poussot,Wang,Solyom,LuGreControl3,deCastro,Fors1,Fors2,Fors3,Fors4}. Indeed, model-based approaches appear ubiquitously in research works focused on the development of Antilock-Braking Systems (ABS) \cite{Nielsen,Savaresi,LuGreControl2,deCastro2,Poussot,Wang}, Traction Control Systems (TCS) \cite{LuGreControl3,Solyom,deCastro}, Electronic Stability Control (ESC) \cite{Fors1,Fors2,Fors3,Fors4,JazarNew}, and friction and sideslip detection algorithms \cite{Doumiati2,Hsu,Mojtaba1,Mojtaba2,Mojtaba3,Shao1,Shao2,Shao3,Shao4,Basilio1,Basilio2,Basilio3,Basilio4,Joa}.

Traditional single-track models \cite{Pacejka2,Guiggiani,Guiggiani2}, widely adopted for their simplicity and analytical tractability \cite{Pacejka2,Savaresi,Shao1,Shao2,Shao3,Shao4}, typically employ quasi-static algebraic expressions for tyre forces, thereby neglecting the dynamic and distributed nature of frictional interactions within the contact patch. Although extensions incorporating tyre relaxation effects via first-order \emph{ordinary differential equations} (ODEs) have gained traction \cite{Higuchi1,Higuchi2,Pauwelussen,SvendeniusA,SvendeniusB,Svendenius3,Rill1,Rill2,PAC,PacejkaBesselink,Two-regime}, these formulations still fall short in capturing spatially distributed transient dynamics, especially under fast steering manoeuvres or low-speed conditions. For instance, spontaneously excited oscillations occurring at low speed, commonly known as \emph{micro-shimmy vibrations}, are not detected by classical lumped tyre models. Their prediction requires more sophisticated formulations based on \emph{partial differential equations} (PDEs), often of hyperbolic type, that can adequately describe wave-like phenomena arising from the distributed nature of the contact patch. The occurrence of micro-shimmy has been confirmed experimentally in \cite{Takacs2,Takacs4}, and primarily studied through infinite-dimensional linear single-track models in \cite{Takacs2,Takacs4,Takacs1,Takacs3,Takacs5,Beregi1,Beregi2,Beregi3,Beregi4,Beregi5,BicyclePDE}, where delay effects due to tyre compliance were modelled using distributed brush-like tyre formulations with and without Coulomb friction.
Crucially, friction-induced force saturation appears to suppress or attenuate micro-shimmy oscillations \cite{Takacs2}. However, in scenarios involving partial sliding, the presence of coexisting stick and slip zones across the contact patch enormously complicates rigorous mathematical treatment, sometimes even posing insurmountable challenges for formal analysis \cite{Sharp}.
In contrast, distributed dynamic friction models, such as the Dahl and LuGre formulations \cite{Dahl,LuGre1}, qualify as desirable alternatives for capturing friction-induced nonlinearities. Indeed, these models do not explicitly distinguish between local stick and slip regimes, yet they retain the essential dissipative and memory properties of the frictional interface. This abstraction not only simplifies theoretical analysis but also offers practical advantages in the design of control and adaptive estimation algorithms, particularly for online identification of tyre-road friction \cite{LuGreControl2,LuGreControl3,Horowitz1,Horowitz2,Horowitz3}. In fact, linear single-track models, even when formulated in terms of distributed parameter systems, are intrinsically inadequate for such purposes, as they lack the necessary nonlinear structure to capture frictional dynamics.

In the context outlined above, this paper aims to bridge the gap by proposing a family of semilinear single-track models with distributed tyre dynamics, built upon the recently introduced \emph{Friction with Bristle Dynamics} (FrBD) model \cite{FrBD}. The FrBD formulation consists of a dynamic friction model that generalises the classical Dahl and LuGre descriptions by representing the tyre-road interaction as a semilinear PDE system incorporating internal micro-damping effects. Moreover, unlike the LuGre model, which is semi-empirical in nature, the FrBD formulation is grounded on solid physical principles, corresponding to a first-order approximation of the implicit dynamics of a bristle element \cite{FrBD,Rill,Rill0}. The proposed approach allows for a detailed description of the internal deformation of the tyre's bristle-like elements across the spatial contact domain, leading to richer transient dynamics and more realistic steady-state behaviour.

The novelty of this work lies in the systematic integration of dynamic, distributed tyre models into a global single-track vehicle framework, resulting in a semilinear hyperbolic ODE-PDE interconnection. Such a system captures the interplay between the vehicle's lateral dynamics and the spatially distributed deformation of the tyre contact patch, considering the effect of friction in a simple yet realistic way. In particular, by adopting the distributed version of the FrBD model, the need to distinguish between local stick-slip phenomena inside the contact patch is completely eliminated, yielding an accurate representation that is mathematically tractable and amenable to formal analysis. More specifically, two variants are developed: a simplified model assuming rigid tyre carcasses and a more refined one accounting for carcass flexibility through additional nonlocal terms. Each model is represented compactly using a state-space formulation, which facilitates mathematical analysis and controller design.

The theoretical contributions of the paper are threefold. First, a wide family of semilinear single-track models is developed, enabling the study of the coupling between the rigid body vehicle motion and the distributed dynamics of the tyres whilst accounting for limited friction via relatively smooth friction models. Second, a rigorous analysis concerning the well-posedness of the resulting ODE-PDE system is conducted, establishing both local and global existence of \emph{mild} and \emph{classical solutions} under reasonable assumptions. Notably, as opposed to the classic LuGre description, the more general FrBD formulation satisfies essential dissipativity and stability properties that permit deducing global well-posedness for nearly every combination of model parameters. Third, a linearisation of the original model is performed around arbitrary equilibria to derive explicit transfer functions and stability conditions via spectral analysis of the associated linear operator. 

The practical impact of the proposed modelling framework is demonstrated through simulations involving micro-shimmy oscillations and transient steering manoeuvres. These examples reveal that distributed models can capture subtle physical effects not observable in lumped formulations. The results underline the relevance of distributed tyre models for advancing the fidelity and predictive power of lateral vehicle dynamics, thereby offering a solid foundation for next-generation automotive control systems. Indeed, whilst the primary focus of this work is on the mechanical modeling of the problem, particular attention is also dedicated to control-oriented aspects, with the ambition of delivering formulations that may be successfully integrated within control and estimation algorithms.

The remainder of this manuscript is organised as follows. Section~\ref{Sect:friction} presents an extensive review of the distributed FrBD formulation, discussing its connection to other friction models available from the literature, including the Dahl and LuGre. The governing ODE-PDEs of the semilinear single-track models are subsequently derived in Sect.~\ref{sect:SemilienarModels}, where well-posedness results are also asserted. Section~\ref{sect:linearLinearised} moves then to investigate the equilibria of the semilinear single-track models, and then to their linearisation. In this way, spectral methods can be effectively applied to study local stability, whilst simultaneously allowing the derivation of a transfer function from the steering input to the lumped state variables. Simulation results are reported instead in Sect.~\ref{sect:Simulations}, concerning the dynamic behaviour of both the semilinear and linear formulations. Finally, Sect.~\ref{sect:Conclusion} concludes the paper by summarising the main findings of the work, and indicating directions for future research. Technical proofs and results are confined to Appendix~\ref{app:theorems}.


\subsection*{Notation}
In this paper, $\mathbb{R}$ denotes the set of real numbers; $\mathbb{R}_{>0}$ and $\mathbb{R}_{\geq 0}$ indicate the set of positive real numbers excluding and including zero, respectively. Similarly, $\mathbb{C}$ is the set of complex numbers; $\mathbb{C}_{>0}$ and $\mathbb{C}_{\geq 0}$ denote the sets of all complex numbers whose real part is larger than and larger than or equal to zero, respectively. 
The set of $n\times m$ matrices with values in $\mathbb{F}$ ($\mathbb{F} = \mathbb{R}$, $\mathbb{R}_{>0}$, $\mathbb{R}_{\geq0}$, or $\mathbb{C}$) is denoted by $\mathbf{M}_{n\times m}(\mathbb{F})$ (abbreviated as $\mathbf{M}_{n}(\mathbb{F})$ whenever $m=n$). For $\mathbb{F} = \mathbb{R}$, $\mathbb{R}_{>0}$, or $\mathbb{R}_{\geq0}$, $\mathbf{GL}_n(\mathbb{F})$ and $\mathbf{Sym}_n(\mathbb{F})$ represents the groups of invertible and symmetric matrices, respectively, with values in $\mathbb{F}$; the identity matrix on $\mathbb{R}^n$ is indicated with $\mathbf{I}_n$. A positive-definite matrix is noted as $\mathbf{M}_n(\mathbb{R}) \ni \mathbf{Q} \succ \mathbf{0}$. The identity operator on a Banach space $\mathcal{Z}$ is denoted by $I_{\mathcal{Z}}$.
The standard Euclidean norm on $\mathbb{R}^n$ is indicated with $\norm{\cdot}_2$; operator norms are simply denoted by $\norm{\cdot}$.
$L^2((0,1);\mathbb{R}^n)$ denotes the Hilbert space of square-integrable functions on $(0,1)$ with values in $\mathbb{R}^n$, endowed with inner product $\langle \bm{\zeta}_1, \bm{\zeta}_2 \rangle_{L^2((0,1);\mathbb{R}^n)} = \int_0^1 \bm{\zeta}_1^{\mathrm{T}}(\xi)\bm{\zeta}_2(\xi) \dif \xi$ and induced norm $\norm{\bm{\zeta}(\cdot)}_{L^2((0,1);\mathbb{R}^n)}$. The Hilbert space $H^1((0,1);\mathbb{R}^n)$ consists of functions $\bm{\zeta}\in L^2((0,1);\mathbb{R}^n)$ whose weak derivative also belongs to $L^2((0,1);\mathbb{R}^n)$; it is naturally equipped with norm $\norm{\bm{\zeta}(\cdot)}_{H^1((0,1);\mathbb{R}^n)}^2 \triangleq \norm{\bm{\zeta}(\cdot)}_{L^2((0,1);\mathbb{R}^n)}^2 + \norm{\pd{\bm{\zeta}(\cdot)}{\xi}}_{L^2((0,1);\mathbb{R}^n)}^2$.
Given a domain $\Omega$ with closure $\overline{\Omega}$, $L^p(\Omega;\mathcal{Z})$ and $C^k(\overline{\Omega};\mathcal{Z})$ ($p, k \in \{1, 2, \dots, \infty\}$) denote respectively the spaces of $L^p$-integrable functions and $k$-times continuously differentiable functions on $\overline{\Omega}$ with values in $\mathcal{Z}$ (for $T = \infty$, the interval $[0,T]$ is identified with $\mathbb{R}_{\geq 0}$). Given two Hilbert spaces $\mathcal{V}$ and $\mathcal{W}$, $\mathscr{L}(\mathcal{V};\mathcal{W})$ denotes the spaces of linear operators from $\mathcal{V}$ to $\mathcal{W}$ (abbreviated $\mathscr{L}(\mathcal{V})$ if $\mathcal{V} = \mathcal{W}$); $\mathscr{B}(\mathcal{V})$ indicates the space of bounded linear operators on $\mathcal{V}$. The spectrum of a possibly unbounded operator $(\mathscr{O},\mathscr{D}(\mathscr{O}))$ with domain $\mathscr{D}(\mathscr{O})$ is denoted by $\sigma(\mathscr{O})$. The group of operators on a Banach space $\mathcal{Z}$ that are infinitesimal generators of a $C_0$-semigroup satisfying $\norm{T(t)} \leq \eu^{\omega t}$ is conventionally denoted by $\mathscr{G}(\mathcal{Z}; 1, \omega)$ \cite{Kato,Tanabe}. The following notation is adapted from \cite{Zwart,CurtainAutomatica}.
Consider $C^\omega(\mathbb{C}_{>0};\mathbf{M}_{n\times m}(\mathbb{C}))$, the set of all analytic functions from $\mathbb{C}_{>0}$ to $\mathbb{C}^{n\times m}$; the Hardy space $H^\infty(\mathbb{C}_{>0};\mathbf{M}_{n\times m}(\mathbb{C}))$ is defined as $H^\infty(\mathbb{C}_{>0};\mathbf{M}_{n\times m}(\mathbb{C})) \triangleq \{ \mathbf{G} \in C^\omega(\mathbb{C}_{>0};\mathbf{M}_{n\times m}(\mathbb{C})) \mathrel{|} \norm{\mathbf{G}(\cdot)}_\infty < \infty\}$, with $\norm{\mathbf{G}(\cdot)}_\infty \triangleq \esup_{ s\in \mathbb{C}_{>0}} \norm{\mathbf{G}(s)}$; more generally, $H^\infty(\mathbb{C}_{>0};\mathcal{Z})$ defines the Hardy space of bounded holomorfic functions on $\mathbb{C}_{> 0}$ with values in $\mathcal{Z}$, that is, $H^\infty(\mathbb{C}_{>0};\mathcal{Z}) \triangleq \{ \mathscr{O} \in C^\omega(\mathbb{C}_{>0};\mathcal{Z}) \mathrel{|} \norm{\mathscr{O}}_\infty < \infty\}$. 
Finally, the Laplace transform of a variable $\bm{\zeta}(t) \in \mathcal{V}$ is denoted by $\widehat{\bm{\zeta}}(s) = (\mathcal{L}\bm{\zeta})(s)$.

\section{A generalised distributed friction model}\label{Sect:friction}
This section introduces a generalised framework to describe rolling contact processes using the distributed FrBD friction model. Specifically, Sect.~\ref{sect:FrictionModel} details the governing equations of the model; closed-form solutions are derived then in Sect.~\ref{sect:Sol}.

\subsection{Model equations}\label{sect:FrictionModel}
The governing equations of the FrBD model include a semilinear PDE describing the evolution of the internal frictional variable, integral functionals to compute the generated frictional force, and appropriate expressions to describe the vertical pressure distribution acting at the interface between the contacting bodies. These equations are reviewed respectively in Sects.~\ref{sect:bristleDyn},~\ref{sect:FF}, and~\ref{sect.vP}.

\subsubsection{Distributed bristle dynamics}\label{sect:bristleDyn}
As illustrated in Fig.~\ref{fig:LuGre}, the distributed FrBD model describes the dynamics of a series of infinitesimal bristles schematising the sliding rolling contact between two bodies over a finite length $L \in \mathbb{R}_{>0}$. Regarding their relative velocity as an input, its governing PDE may be interpreted either as a semilinear or as a linear non-autonomous one. Without loss of generality, it may be cast on a unit domain in the Eulerian reference frame: 
\begin{subequations}\label{eq:FrBDPDE00}
\begin{align}
\begin{split}
& \dpd{z(\xi,t)}{t} + V\dpd{z(\xi,t)}{\xi} = -\dfrac{\sigma_{0} \abs{v(t)}_\varepsilon}{g\bigl(v(t);\chi_1\bigr)}z(\xi,t) + \dfrac{\mu\bigl(v(t)\bigr)}{g\bigl(v(t);\chi_1\bigr)} v(t), \quad (\xi,t) \in (0,1)\times(0,T), 
\end{split} \label{eq:dynZ} \\
 & z(0,t)  = 0, \quad t \in (0,T),  \label{eq:BCsjdjs}
\end{align}
\end{subequations}
where $z(\xi,t) \in \mathbb{R}$ is the distributed variable describing the deformation of a bristle located at the coordinate $\xi \in [0,1]$ at time $t \in [0,T]$, $v(t) \in \mathbb{R}$ denotes the (rigid) relative velocity between the bodies, which is supposed to be independent of the spatial coordinate, $V\in \mathbb{R}_{>0}$ indicates the transport velocity (which coincides with the rolling speed rescaled on the unit domain), $\sigma_0 \in \mathbb{R}_{>0}$ is the \emph{normalised micro-stiffness} coefficient, $\mu\in C^0(\mathbb{R};[\mu\ped{min},\infty))$, with $\mu\ped{min}\in \mathbb{R}_{>0}$, is an expression for the friction coefficient at the interface between the two contacting bodies in relative rolling motion, $\abs{\cdot}_\varepsilon \in C^0(\mathbb{R};\mathbb{R}_{\geq 0})$ denotes the (possibly regularised\footnote{It is common in engineering practice to replace the absolute value with differentiable functions \cite{Rill,Rill0}, such as $\abs{v}_\varepsilon \triangleq \sqrt{v^2 + \varepsilon}$, for some $\varepsilon \in \mathbb{R}_{>0}$. This paper also considers the case $\varepsilon \in \mathbb{R}_{\geq 0}$, thus generalising the common formulations of the Dahl, LuGre, and FrBD models.}) absolute value, and
\begin{align}\label{eq:gij0}
g(v;\chi_{1}) \triangleq \chi_1\sigma_{1} \abs{v}_\varepsilon + \mu(v),
\end{align}
where $\sigma_1 \in \mathbb{R}_{\geq 0}$ indicates the \emph{normalised micro-damping} coefficient, and $\chi_1 \in \{0,1\}$ is a parameter introduced to distinguish between different parametrisations of the FrBD model. Indeed, Eq.~\eqref{eq:FrBDPDE00} reduces to the PDEs obtained for the LuGre and Dahl models whenever the term $\chi_1 = 0$, which is equivalent to neglecting the damping term in Eq.~\eqref{eq:gij0}; for $\chi_1 = 1$, it yields instead the full FrBD formulation. As later revealed in Sect.~\ref{sect:well}, incorporating the damping term in Eq.~\eqref{eq:gij0} has important repercussions on the behaviour of the model, permitting to infer global well-posedness under very mild hypotheses. In a similar context, when the rigid relative velocity $v(t)$ is regarded as an input, the PDE~\eqref{eq:FrBDPDE00} is well-posed and input-to-state stable, according to the analyses conducted in \cite{FrBD,MScthesis,DistrLuGre}. Moreover, dissipativity holds virtually for any combination of constant model parameters \cite{DistrLuGre,FrBD}. For the sake of simplicity, this paper restricts itself to the case of constant coefficients, which is common in tyre and vehicle dynamics.

Before moving to Sect.~\ref{sect:FF}, some considerations are in order. \emph{In primis}, it is worth stressing that the left-hand side of~\eqref{eq:dynZ} coincides with the total time derivative, namely
\begin{align}\label{eq:totalcjses}
\dod{z(\xi,t)}{t} = \dpd{z(\xi,t)}{t} + V\dpd{z(\xi,t)}{\xi},
\end{align}
where $\od{z(\xi,t)}{t} \in \mathbb{R}$ represents the Lagrangian time derivative from the perspective of a material particle travelling inside the contact area. 

\emph{In secundis}, the PDE~\eqref{eq:FrBDPDE00} is postulated over the entire unit domain. Analogously to the Dahl and LuGre friction models, the FrBD formulation does not explicitly differentiate between local adhesion and sliding conditions within the contact area. As a consequence, the variable $z(\xi,t)$ is required to vanish only at the leading edge, which leads to enforcing the natural boundary condition (BC)~\eqref{eq:BCsjdjs}. This contrasts with classical Coulomb-type friction models, in which multiple stick and slip zones may coexist within the contact region, and where the use of compactly supported vertical pressure distributions, together with the assumption of a linear elastic behavior of the bristle element, implies zero deformation also at the trailing edge \cite{LibroMio,Guiggiani,Pacejka2}. 
Although the FrBD description is less phenomenologically accurate than the classical Coulomb formulation, it is nonetheless physically justified as a first-order dynamic approximation of a rheological model describing the bristle element \cite{FrBD,Rill,Rill0}. Moreover, it can be more readily integrated within complex mechanical system models, which motivates its adoption in this work.

Finally, it should be observed that Eq.~\eqref{eq:FrBDPDE00} merely describes a scalar distributed friction model, which may be applied to any rolling contact system. More specifically, a schematic of the FrBD formulation particularised for an accelerating tyre is illustrated in Fig.~\ref{fig:LuGre2}, where the relevant kinematic quantities are also shown. In Fig.~\ref{fig:LuGre2}, the variable $z(\xi,t)$, defined positive in the opposite direction of the coordinate $\xi$, should be interpreted as the longitudinal deflection of a bristle element, with the (rigid) relative velocity given by $v(t) = V\ped{r}(t)-V_x(t)$, where $V\ped{r}(t) \in \mathbb{R}_{>0}$ denotes the (free) rolling speed of the tyre \cite{Guiggiani}, and $V_x(t) \in \mathbb{R}_{>0}$ its longitudinal velocity.

\begin{figure}
\centering
\includegraphics[width=0.5\linewidth]{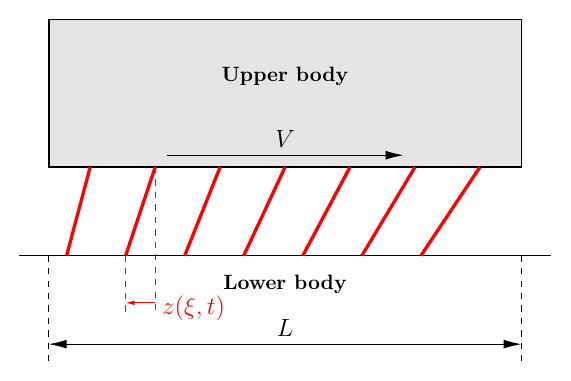} 
\caption{A schematic representation of the distributed FrBD model.}
\label{fig:LuGre}
\end{figure}

\begin{figure}
\centering
\includegraphics[width=0.5\linewidth]{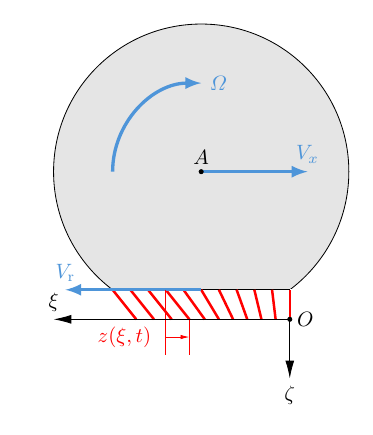} 
\caption{Schematic of the distributed FrBD friction model particularised for a tyre-wheel system during acceleration. The wheel travels in longitudinal direction with speed $V_x\in \mathbb{R}_{>0}$, and has angular velocity $\Omega \in \mathbb{R}_{>0}$. A local coordinate frame $(O;\xi,\eta,\zeta)$ is attached to the tyre,  with the origin $o$ coincident with the leading-edge point.  The coordinates $\xi$ and $\eta$ lie in the road plane, whilst $\zeta$ points downward (into the road). The longitudinal coordinate $\xi$ is aligned with the transport velocity $V = \dfrac{V\ped{r}}{L}$, where $V\ped{r} \in \mathbb{R}_{>0}$ denotes the free-rolling speed of the tyre \cite{Guiggiani}. The $\eta$-axis is oriented so that the local frame $(O;\xi,\eta,\zeta)$ is left-handed.}
\label{fig:LuGre2}
\end{figure}


\subsubsection{Frictional force}\label{sect:FF}
Starting with the PDE~\eqref{eq:FrBDPDE00}, the total frictional force generated during the rolling contact process may be calculated as \cite{TsiotrasConf}
\begin{align}\label{eq:Forces10}
\begin{split}
F(t) & = L\int_0^1 p(\xi)\biggl(\sigma_{0}z(\xi,t) + \sigma_{1}\dod{z(\xi,t)}{t} + \sigma_{2}v(t) \biggr) \dif \xi,
\end{split}
\end{align}
where $\sigma_2 \in \mathbb{R}_{\geq 0}$ is the normalised \emph{viscous damping} coefficient, and $p \in C^1([0,1];\mathbb{R}_{\geq 0})$ denotes the vertical pressure distribution acting inside the contact area. The total derivative appearing in the above Eq.~\eqref{eq:Forces10} formally corresponds to the left-hand side of Eq.~\eqref{eq:FrBDPDE00}. In some works \cite{Tsiotras1,Tsiotras2,Deur0,Deur1,Deur2}, it is replaced by the partial one, yielding the alternative formula
\begin{align}\label{eq:Forces10part}
\begin{split}
F(t) & = L\int_0^1 p(\xi)\biggl(\sigma_{0}z(\xi,t) + \sigma_{1}\dpd{z(\xi,t)}{t} + \sigma_{2}v(t) \biggr) \dif \xi.
\end{split}
\end{align}
By noting from Eq.~\eqref{eq:totalcjses} that
\begin{align}
\dpd{z(\xi,t)}{t} = \dod{z(\xi,t)}{t}-V\dpd{z(\xi,t)}{\xi},
\end{align}
a general expression for the total frictional force may be introduced as
\begin{align}\label{eq:Forces10part2}
\begin{split}
F(t) & = L\int_0^1 p(\xi)\Biggl[\sigma_{0}z(\xi,t) + \sigma_{1}\biggl(\dod{z(\xi,t)}{t}-\chi_2V\dpd{z(\xi,t)}{\xi}\biggr) + \sigma_{2}v(t) \Biggr] \dif \xi,
\end{split}
\end{align}
where the parameter $\chi_2 \in \{0,1\}$ allows easily between the formulations in Eqs.~\eqref{eq:Forces10} and~\eqref{eq:Forces10part}, with $\chi_2 = 0$ if the total time derivative is employed to model the damping term, and $\chi_2 = 1$ if the partial time derivative is used. As corroborated by the findings of \cite{DistrLuGre}, employing the partial time derivative in Eq.~\eqref{eq:Forces10part2} (that is, setting $\chi_2 = 1$), might introduce some inconsistencies compared to the original mathematical properties of the lumped LuGre ($\chi_1 = 0$) and FrBD ($\chi_1 = 1$) models. Besides, the partial derivative does not have a clear physical meaning, whereas the total one represents a real deformation velocity. In the same context, it should be mentioned that the use of the partial time derivative also deteriorates the stability estimates for the frictional force that are immediately available when using the total one. This aspect will be further clarified in Sect.~\ref{sect:well}. In spite of these considerations, both variants are considered in the present manuscript, since using the partial derivative has become somewhat standard in the dedicated literature.

\subsubsection{Vertical pressure distribution}\label{sect.vP}
The computation of the total frictional force demands the specification of a function describing the contact pressure distribution. Introducing the total vertical force acting on the upper body, $F_z\in \mathbb{R}_{>0}$, suitable choices include the constant one
\begin{align}\label{eq:pConst}
p(\xi) = p_0 = \dfrac{F_z}{L},
\end{align}
the exponentially decreasing one
\begin{align}\label{eq:pExp}
p(\xi) = p_0\exp(-a\xi) = \dfrac{F_z}{L}\dfrac{a\exp(-a\xi)}{1-\exp(-a)},
\end{align}
for some appropriate $a\in \mathbb{R}_{>0}$, and the parabolic one
\begin{align}\label{eq:pParab}
p(\xi) = p_0\xi(1-\xi) = \dfrac{6F_z}{L}\xi(1-\xi).
\end{align}
Clearly, a generic pressure distribution must satisfy
\begin{align}\label{eq:pInt}
L\int_0^1 p(\xi) \dif \xi = F_z.
\end{align}
The next Sect.~\ref{sect:Sol} derives the general, closed-form solution to the PDE~\eqref{eq:FrBDPDE00}, providing also some expressions for the total frictional force generated in steady-state.

\subsection{Model solution}\label{sect:Sol}
Treating the relative velocity $v(t)$ as an exogenous input to the FrBD model, the PDE~\eqref{eq:FrBDPDE00} may be solved explicitly, allowing for the retrieval of analytical relationships describing the steady-state frictional forces. In particular, a general solution to the PDE~\eqref{eq:FrBDPDE00} is reported in Sect.~\ref{sect:bristleSol}, whereas Sect.~\ref{sect:ssFrcoes} provides some closed-form expressions for the stationary frictional force.

\subsubsection{Bristle dynamics}\label{sect:bristleSol}
In most cases of practical interest, the PDE~\eqref{eq:FrBDPDE00} may be solved in closed form using the method of the characteristic lines, as explained in \cite{FrBD}. The complete transient solution is reported, for instance, in \cite{FrBD}, but omitted here for brevity. In contrast, of particular importance for the results advocated in this work is the stationary solution, which reads
\begin{align}\label{eq:CD2}
z(\xi) = \sgn_\varepsilon(v)\dfrac{\mu(v)}{\sigma_0}\Biggl[1-\exp\biggl( -\dfrac{\sigma_0\abs{v}_\varepsilon}{V g(v;\chi_1)}\xi\biggr)\Biggr], \quad \xi \in [0,1],
\end{align}
where
\begin{align}
\sgn_\varepsilon(v) \triangleq \dfrac{v}{\abs{v}_\varepsilon}.
\end{align}
Equation~\eqref{eq:CD2} is consistent with the expression deduced according to the distributed LuGre model (with $\chi_1 = 0$). Figure~\ref{fig:rectGough}(a) illustrates the trend of the bristle deflection $z(\xi)$ calculated according to~\eqref{eq:CD2} for three different values of relative velocities $v = 1$, 5, and 10 $\textnormal{m}\,\textnormal{s}^{-1}$ and for $V = 20$ $\textnormal{m}\,\textnormal{s}^{-1}$, with $\chi_1 = 1$. In particular, it is interesting to observe that the maximum deflection decreases with $v$, whereas the partial space derivative exhibits an opposite trend. In turn, this implies that the bristle deformation converges faster to its asymptotic value as $v$ increases. The trends depicted in Fig.~\ref{fig:rectGough}(a) were produced with the parameter values listed in Table~\ref{tab:param1}, which are typical of those encountered in the literature (see \cite{FrBD} for details). In particular, the friction coefficient was modelled using a generalised Coulomb formulation of the type
\begin{align}\label{eq:fViscOOO}
\mu(v) = \mu\ped{d} + (\mu\ped{s}-\mu\ped{d})\eu^{-(\abs{v}/v\ped{S})^2} +\sigma_3v,
\end{align}
where $\mu\ped{d}, \mu\ped{s} \in \mathbb{R}_{>0}$ denote the \emph{dynamic} and \emph{static friction coefficient}, respectively, $v\ped{S}$ is the \emph{Stribeck velocity}, and $\sigma_3 \in \mathbb{R}_{\geq 0}$ is the normalised \emph{viscous friction} coefficient.

\begin{figure}
\centering
\subfloat[Bristle deflection $z(\xi)$ along the contact area for three different values of relative velocities $v  =1$, 5, and 10 $\textnormal{m}\,\textnormal{s}^{-1}$.]{%
\resizebox*{7.5cm}{!}{\includegraphics{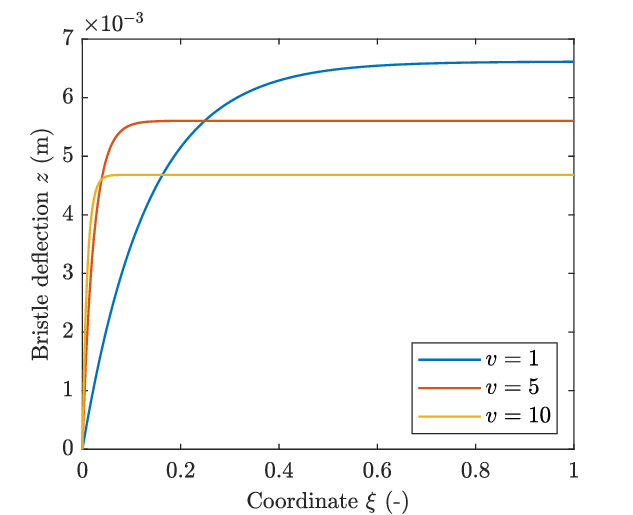}}}\hspace{5pt}
\subfloat[Frictional force $F$ for the constant pressure distribution (blue line), and exponential one (orange line), for a total vertical load of $F_z = 3000$ N.]{%
\resizebox*{7.5cm}{!}{\includegraphics{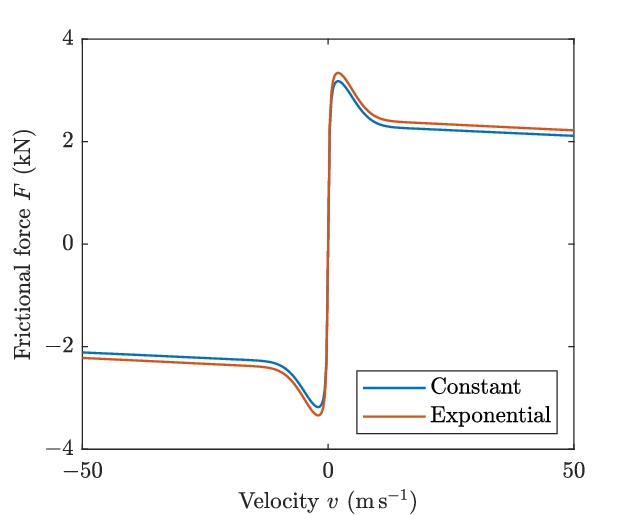}}}\hspace{5pt}
\caption{Bristle deflection and frictional force predicted according to the FrBD model, with $\chi_1 = 1$ and $\chi_2 = 0$. Model parameters as in Table~\ref{tab:param1}.} \label{fig:rectGough}
\end{figure}

\subsubsection{Frictional force}\label{sect:ssFrcoes}
Starting with Eq.~\eqref{eq:CD2}, an analytical expression for the total frictional force may be deduced by appropriately specifying the function $p(\cdot)$. In particular, assuming a constant pressure distribution as in Eq.~\eqref{eq:pConst} gives
\begin{align}\label{eq:Fss1}
\begin{split}
F &= \dfrac{vLp_0V\bar{\sigma}_0(v;\chi_1)\mu(v)g(v;\chi_1)}{\sigma_0^2 \abs{v}_\varepsilon^2}\Biggl[\dfrac{\sigma_0\abs{v}_\varepsilon}{V g(v;\chi_1)}+\exp\biggl( -\dfrac{\sigma_0\abs{v}_\varepsilon}{V g(v;\chi_1)}\biggr)-1\Biggr]  + Lp_0\bar{\sigma}_2(v;\chi_1)v \\
& \quad - \chi_2\sgn_{\varepsilon}(v)\dfrac{Lp_0V\sigma_1\mu(v)}{\sigma_0}\Biggl[1-\exp\biggl( -\dfrac{\sigma_0\abs{v}_\varepsilon}{V g(v;\chi_1)}\biggr)\Biggr],
\end{split}
\end{align}
whereas the exponentially decreasing distribution in Eq.~\eqref{eq:pExp} yields
\begin{align}\label{eq:Fss2}
\begin{split}
F&= \dfrac{\sgn_\varepsilon(v)Lp_0V\mu(v)g(v;\chi_1)}{\sigma_0\bigl(aV g(v;\chi_1)+\sigma_0\abs{v}_\varepsilon\bigr)}\bigl(\bar{\sigma}_0(v;\chi_1)-\chi_2aV\sigma_1\bigr) \Biggl[\exp\biggl( -\dfrac{aV g(v;\chi_1)+\sigma_0\abs{v}_\varepsilon}{V g(v;\chi_1)}\biggr)-1\Biggr] \\
& \quad + \dfrac{Lp_0\bigl(1-\exp(-a)\bigr)}{a}\biggl[\sgn_\varepsilon(v)\bigl(\bar{\sigma}_0(v;\chi_1)-\chi_2aV\sigma_1\bigr)\dfrac{\mu(v)}{\sigma_0}+ \bar{\sigma}_2(v;\chi_1)v\biggr] \\
& \quad - \chi_2\sgn_\varepsilon(v)\dfrac{Lp_0V\exp(-a) \mu(v)}{\sigma_0}\Biggl[1-\exp\biggl( -\dfrac{\sigma_0\abs{v}_\varepsilon}{V g(v;\chi_1)}\biggr)\Biggr],
\end{split}
\end{align}
where
\begin{subequations}
\begin{align}
\bar{\sigma}_0(v; \chi_1) &\triangleq \sigma_0\left(1-\dfrac{\sigma_1 \abs{v}_\varepsilon}{g(v;\chi_1)}\right), \\
\bar{\sigma}_2(v;\chi_1) & \triangleq \sigma_2 + \sigma_1\dfrac{\mu(v)}{g(v;\chi_1)}.
\end{align}
\end{subequations}
Both the expressions~\eqref{eq:Fss1} and~\eqref{eq:Fss2}, which appear to be novel in the literature, are in theoretical agreement with the corresponding ones deduced using the distributed LuGre model (for $\chi_1 = 0$). For $\chi_2 = 0$, the relationship between the steady-state bristle force $F$ and the relative velocity $v$ may be visualised in Fig.~\ref{fig:rectGough}(b) for $V = 20$ $\textnormal{m}\,\textnormal{s}^{-1}$ concerning the constant and exponential pressure distributions. The observed behavior is coherent with the trend of the bristle deformation $z(\xi)$ in stationary conditions, as depicted in Fig.~\ref{fig:rectGough}(a): as the relative velocity increases in magnitude, the friction characteristic first saturates, and then rapidly decreases. The parabolic profile yields a rather involved expression for the frictional force, which is omitted for brevity.

\begin{table}[h!]\centering 
\caption{Model parameters}
{\begin{tabular}{|c|c|c|c|}
\hline
Parameter & Description & Unit & Value \\
\hline 
$L$ & Contact length & m & 0.1 \\
$\sigma_0$ & Normalised micro-stiffness & $\textnormal{m}^{-1}$ & 180 \\
$\sigma_1$ & Normalised micro-damping & $\textnormal{s}\,\textnormal{m}^{-1}$ &0 \\
$\sigma_2$ & Normalised viscous damping & $\textnormal{s}\,\textnormal{m}^{-1}$ &0 \\
$\sigma_3$ & Normalised viscous friction & $\textnormal{s}\,\textnormal{m}^{-1}$ &0.0018 \\
$\mu\ped{d}$ & Dynamic friction coefficient & - &0.8 \\
$\mu\ped{s}$ & Static friction coefficient & - &1.2 \\
$v\ped{S}$ & Stribeck velocity & $\textnormal{m}\,\textnormal{s}^{-1}$ &0.6 \\
$a$ & Pressure parameter & - & 0.1 \\
$\varepsilon$ & Regularisation parameter & $\textnormal{m}^2\,\textnormal{s}^{-2}$ & $0$ \\
\hline
\end{tabular} }
\label{tab:param1}
\end{table}

\section{Semilinear single-track models}\label{sect:SemilienarModels}
Adopting the general FrBD formulation introduced in Sect.~\ref{Sect:friction}, it is possible to derive a family of semilinear single-track models describing the lateral dynamics of the vehicle when the tyres operate in a restricted range of the nonlinear region. In the following, Sect.~\ref{sect:SingleDer} details the complete set of equations governing the planar motion of the vehicle. A general state-space representation, more amenable to mathematical analysis, is then provided in Sect.~\ref{sect:stateSpace}. Finally, Sect.~\ref{sect:well} studies the well-posedness of the derived formulations.

\subsection{Model equations}\label{sect:SingleDer}
The governing equations of the model include those describing the global equilibrium of the single-track vehicle (Sect.~\ref{sect:glovalEq}), the distributed tyre dynamics (Sect.~\ref{esct:distrYr}), and the rigid relative velocities (Sect.~\ref{sect:rigidRelVel}). Before presenting these in detail, it is important to outline the key assumptions underlying their derivation: namely, small steering angles and negligible lateral load transfers. The latter assumption, in particular, implies small lateral accelerations, which limits the applicability of the proposed single-track formulations to moderate dynamic conditions. However, unlike many classical approaches in the literature, the present framework does not assume that the tyres operate within their linear range. This distinction enables the analysis of nonlinear friction-induced effects, even under relatively gentle manoeuvres. As such, the models developed here are especially well suited to capture transient lateral dynamics in low-friction conditions or moderate-speed operations, thereby extending the scope of traditional linear single-track models.

\subsubsection{Global equilibrium}\label{sect:glovalEq}
For sufficiently small steering angles $\delta_1(t), \delta_2(t) \in \mathbb{R}$, the global equilibrium equations governing the lateral dynamics of the single-track model may be deduced as \cite{Guiggiani}
\begin{subequations}\label{eq:bycicle}
\begin{align}
 \dot{v}_y(t)  & = -\dfrac{1}{m}\bigl(F_{y1}(t) + F_{y2}(t)\bigr)-v_xr(t),  && \\
\dot{r}(t) & =-\dfrac{1}{I_z}\bigl( l_1 F_{y1}(t)-l_2F_{y2}(t)\bigr), && t \in (0,T),
\end{align}
\end{subequations}
where the state vector $[v_y(t)\; r(t)]^{\mathrm{T}} \in \mathbb{R}^2$ contains the lateral velocity of the vehicle's centre of gravity and its yaw rate, respectively, $v_x \in \mathbb{R}_{>0}$ denotes its constant longitudinal speed,  $m \in \mathbb{R}_{>0}$ is the total mass, $I_z \in \mathbb{R}_{>0}$ the moment of inertia of the centre of gravity around the vertical axis, $l_1, l_2 \in \mathbb{R}_{>0}$ the distances from the centre of gravity to the front and rear tyres' contact patch centres (that is, the axle lengths), and $F_{yi}(t) \in \mathbb{R}$, $i \in \{1,2\}$, the total lateral forces generated by the tyres and acting at the front and rear axle. A schematic representation of the single-track model is illustrated in Fig.~\ref{figureForcePostdoc}.

\begin{figure}
\centering
\includegraphics[width=0.9\linewidth]{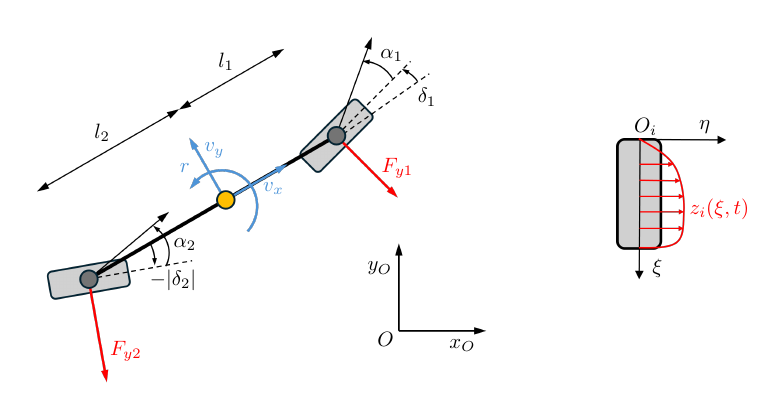} 
\caption{Left: Schematic of the single track model, with its kinematic (blue), dynamic (red), and geometric (black) variables; right: schematic of a single tyre with corresponding reference frame and lateral bristle deflection (in red).}
\label{figureForcePostdoc}
\end{figure}

\subsubsection{Distributed tyre dynamics}\label{esct:distrYr}
As mentioned in Sect.~\ref{sect:bristleDyn}, the scalar PDE~\eqref{eq:FrBDPDE00} can describe several rolling contact systems. In particular, whereas the situation illustrated in Fig.~\ref{fig:LuGre2} deals with the longitudinal case, Eq.~\eqref{eq:FrBDPDE00} may also be employed to describe the lateral dynamics of the tyre, if the variable $z(\xi,t)$ is interpreted as the lateral deflection of a bristle element (see Fig.~\ref{figureForcePostdoc}), and all the involved quantities are appropriately specified in the lateral direction \cite{Deur1,Deur2}. In particular, when the vehicle drives at a constant longitudinal speed $v_x$, in the absence of longitudinal slip, the free rolling speed of the tyres amounts approximately to $V\ped{r}\approx v_x$. Consequently, coherently with the modelling approach of Sect.~\ref{sect:FrictionModel}, the distributed dynamics of each axle may be captured by the following semilinear PDE, postulated on a unit domain:
\begin{subequations}\label{eq:FrBD}
\begin{align}
& \dpd{z_{i}(\xi,t)}{t} + \dfrac{v_x}{L_{i}}\dpd{z_{i}(\xi,t)}{\xi} = -\dfrac{\sigma_{0,i} \abs{v_{i}(t)}_\varepsilon}{g_{i}\bigl(v_{i}(t);\chi_1\bigr)}z_{i}(\xi,t) + \dfrac{2\mu_{i}\bigl(v_{i}(t)\bigr)}{g_{i}\bigl(v_{i}(t);\chi_1\bigr)}v_{i}(t), \quad (\xi,t) \in (0,1)\times(0,T), \label{eq:FrBDPDE} \\
 & z_{i}(0,t)  = 0, \quad t \in (0,T), 
\end{align}
\end{subequations}
where all the symbols have the same meaning as in Sect.~\ref{sect:FrictionModel}, but refer to the axle $i \in \{1,2\}$ via the corresponding subscripts, and 
\begin{align}\label{eq:gij}
g_{i}(v;\chi_1) \triangleq \chi_1\sigma_{1,i} \abs{v}_\varepsilon + \mu_{i}(v).
\end{align}
It is worth clarifying that the variable $z_i(\xi,t) \in \mathbb{R}$, $i\in \{1,2\}$, represents the sum of the lateral deformations of the bristles attached to the left and right tyres mounted on the same axle, which justifies the factor of two multiplying the input on the right-hand side of~\eqref{eq:FrBD}. In contrast, the other model parameters refer to the individual tyre of the axle. As for the distributed FrBD formulation, the PDE~\eqref{eq:FrBD} encompasses all the friction models discussed in Sect.~\ref{sect:FrictionModel}, which are ultimately equivalent when $\sigma_{1,i} =  0$. However, in many cases of practical interest, the effect of the micro-damping and viscous terms on the transient dynamics of the tyre is minor, and may be fairly neglected. On the contrary, relaxation phenomena originating from the compliance of the tyre carcass may play a crucial role. To this end, the following variant of the PDE~\eqref{eq:FrBD}, derived in \cite{CarcassDyn}\footnote{In \cite{CarcassDyn}, more complex versions were derived accounting for the contributions of the damping and viscous terms in the computation of the friction force; however, the resulting increase in complexity -- with an additional nonlinear ODE required to describe the evolution of the transient tyre force -- makes the general formulation less appealing for studies concerned with full vehicle dynamics. Besides, the micro-damping and viscous coefficients are always found to be very small, if not zero, in the literature (see, e.g., the values reported in \cite{Tsiotras1,Tsiotras2,Deur0,Deur1,Deur2}).}, is also considered:
\begin{align}\label{eq:ooCarcass}
\begin{split}
& \dpd{z_{i}(\xi,t)}{t} + \dfrac{v_x}{L_{i}}\dpd{z_{i}(\xi,t)}{\xi} = -\dfrac{\sigma_{0,i} \abs{v_{i}(t)}_\varepsilon}{\mu_{i}\bigl(v_{i}(t)\bigr)}\Biggl(z_{i}(\xi,t)-\psi_{i}\int_0^1 \bar{p}_{i}(\xi)z_{i}(\xi,t)\dif \xi \Biggr) \\
& \qquad \qquad \qquad \qquad \qquad \qquad \quad +v_x\dfrac{\psi_{i}}{L_i}\Biggl( \bar{p}_{i}(1)z_{i}(1,t)-\int_0^1 \dod{\bar{p}_{i}(\xi)}{\xi}z_{i}(\xi,t)\dif\xi\Biggr) + 2\phi_{i}v_{i}(t), \\
& \qquad \qquad \qquad \qquad \qquad \qquad \quad \quad (\xi,t) \in (0,1)\times(0,T), 
\end{split}\\
 & z_{i}(0,t)  = 0, \quad t \in (0,T), 
\end{align}
with the nondimensional pressure distribution $\bar{p}_i \in C^1([0,1];\mathbb{R}_{\geq 0})$, $i\in \{1,2\}$, defined as
\begin{align}
\bar{p}_i(\xi) \triangleq \dfrac{L_i}{F_{zi}}p_i(\xi), \quad i \in \{1,2\},
\end{align}
and
\begin{subequations}\label{eq:etaPsi}
\begin{align}
\phi_{i} & \triangleq \dfrac{w_{i}}{\sigma_{0,i}F_{zi} + w_{i}}, \\
\psi_{i} & \triangleq \dfrac{\sigma_{0,i}F_{zi}}{\sigma_{0,i}F_{zi} + w_{i}}, \quad i\in \{1,2\},
\end{align}
\end{subequations}
where $w_i \in \mathbb{R}_{>0}$, $i \in \{1,2\}$, denotes the lateral stiffness of the tyre carcass.
It is crucial to observe that, according to Eqs.~\eqref{eq:etaPsi}, $\phi_i,\psi_i \in (0,1)$, $i\in\{1,2\}$, which plays a crucial role in determining both the global well-posedness and the stability of the model. Furthermore, as might be intuitively expected, lengthy but straightforward manipulations confirm that the steady-state solution of Eq.~\eqref{eq:ooCarcass} coincides with that of the simpler PDE~\eqref{eq:FrBD}.

As for the distributed FrBD model described in Sect.~\ref{Sect:friction}, the tyre forces on each axle are computed according to
\begin{align}\label{eq:Forces1}
\begin{split}
F_{yi}(t) & = L_{i}\int_0^1 p_{i}(\xi)\Biggl[\sigma_{0,i}z_{i}(\xi,t) + \sigma_{1,i}\biggl(\dod{z_{i}(\xi,t)}{t}-\dfrac{\chi_2v_x}{L_i}\dpd{z_i(\xi,t)}{\xi}\biggr) + 2\sigma_{2,i}v_{i}(t) \Biggr] \dif \xi \\
& = F_{zi}\int_0^1 \bar{p}_{i}(\xi)\Biggl[\sigma_{0,i}z_{i}(\xi,t) + \sigma_{1,i}\biggl(\dod{z_{i}(\xi,t)}{t}-\dfrac{\chi_2v_x}{L_i}\dpd{z_i(\xi,t)}{\xi}\biggr) + 2\sigma_{2,i}v_{i}(t) \Biggr] \dif \xi, \quad i \in \{1,2\},
\end{split}
\end{align}
where the factor of two multiplying the rigid relative velocity appears because $F_{yi}(t)$, $i\in \{1,2\}$, describe the total axle forces, and not those generated by the single tyres. Obviously, when the PDE~\eqref{eq:ooCarcass} is employed, $\sigma_{1,i} = \sigma_{2,i} = 0$, $i \in \{1,2\}$, in Eq.~\eqref{eq:Forces1}.

Finally, different expressions for the nondimensional contact pressure may be postulated. Starting with Eqs.~\eqref{eq:pConst},~\eqref{eq:pExp}, and~\eqref{eq:pParab}, it is possible to deduce the relationship
\begin{align}\label{eq:pConstNorm}
\bar{p}_i(\xi) = \bar{p}_{0,i} = 1, 
\end{align}
for the constant one,
\begin{align}
\bar{p}_i(\xi) = \bar{p}_{0,i}\eu^{-a_i\xi} = \dfrac{a_i\eu^{-a_i\xi}}{1-\eu^{-a_i}},
\end{align}
with $a_i\in \mathbb{R}_{>0}$, for the exponentially decreasing one, and
\begin{align}
\bar{p}_i(\xi) = \bar{p}_{0,i}\xi(1-\xi) = 6\xi(1-\xi)
\end{align}
for the parabolic one.
Finally, it is worth noting that Eq.~\eqref{eq:pInt} identically implies
\begin{align}
\int_0^1 \bar{p}_i(\xi)\dif \xi = 1.
\end{align}

Well-posedness for the PDE~\eqref{eq:ooCarcass} has been established in \cite{MScthesis} invoking classic results available for non-autonomous equations. Stability properties may be proved for certain combinations of model parameters by resorting to standard Lyapunov arguments. However, energy estimates obtained by applying this procedure are inevitably conservative. In practice, stability must be verified numerically for the specific system under consideration. Alternatively, spectral methods may be applied after linearisation of~\eqref{eq:ooCarcass}. For instance, when $\bar{p}(\xi)$ reads as in Eq.~\eqref{eq:pConstNorm}, linearising around small deformations and rigid relative velocities yields the linear full contact patch model analysed in \cite{MioSolo}. Concerning instead the dissipative behaviour of the system described by Eqs.~\eqref{eq:ooCarcass}-\eqref{eq:Forces1}, it is essential to emphasise that the PDE~\eqref{eq:ooCarcass} does not simply describe a friction model, but a mechanical system with friction. Therefore, passivity should not be regarded as an intrinsic property, although a dissipative behaviour might be expected from the interconnected tyre-carcass system.

\subsubsection{Rigid relative velocities and apparent slip angles}\label{sect:rigidRelVel}
A last set of equations is required to link the expressions for the rigid relative velocities appearing in Eqs.~\eqref{eq:FrBD},~\eqref{eq:ooCarcass}, and~\eqref{eq:Forces1} with the kinematic variables describing the lateral motion of the vehicle. Specifically, the rigid relative velocities $v_{i}(t)$ are related to the apparent slip angles $\alpha_{i}(t)$, $i \in \{1,2\}$, via the following relationships \cite{Guiggiani}: 
\begin{align}
v_{i}(t) = v_x\alpha_{i}(t),
\end{align}
where, with some abuse of notation,
\begin{subequations}\label{eq:slipAngles}
\begin{align}
\alpha_{1}(t) &= \alpha_{1}\bigl(v_y(t),r(t),\delta_{1}(t)\bigr) \approx \dfrac{v_y(t) + l_1r(t)}{v_x}-\delta_{1}(t), \\
\alpha_{2}(t) &= \alpha_{2}\bigl(v_y(t),r(t),\delta_{2}(t)\bigr) \approx \dfrac{v_y(t)-l_2r(t)}{v_x}-\chi_3\delta_{2}(t),
\end{align}
\end{subequations}
where the parameter $\chi_3 \in \{0,1\}$ describes the actuation at the rear wheels.
Equations~\eqref{eq:bycicle}-\eqref{eq:slipAngles} permit the derivation of several different semilinear single-track formulations, depending on the adopted models of friction and tyre carcass. The next Sect.~\ref{sect:stateSpace} provides a compact, unified representation for this family of single-track models.

\subsection{State-space representation}\label{sect:stateSpace}
In order to analyse the dynamical behaviour of the single-track models derived in Sect.~\ref{sect:SingleDer}, it is beneficial to introduce a general state-space representation covering both the variants with rigid and flexible carcass. To this end, the following semilinear ODE-PDE system is considered:
\begin{subequations}\label{eq:originalSystems}
\begin{align}
& \dot{\bm{x}}(t) = \mathbf{A}_1\bm{x}(t) + \mathbf{G}_1\biggl[(\mathscr{K}_1\bm{z})(t) + \mathbf{\Sigma}\Bigl(\bm{v}\bigl(\bm{x}(t),\bm{\delta}(t)\bigr)\Bigr)(\mathscr{K}_2\bm{z})(t)+\bm{h}_1\Bigl(\bm{v}\bigl(\bm{x}(t),\bm{\delta}(t)\bigr)\Bigr)\biggr], \quad t \in (0,T), \\
\begin{split}
& \dpd{\bm{z}(\xi,t)}{t} + \mathbf{\Lambda}\dpd{\bm{z}(\xi,t)}{\xi} = \mathbf{\Sigma}\Bigl(\bm{v}\bigl(\bm{x}(t),\bm{\delta}(t)\bigr)\Bigr)\bigl[\bm{z}(\xi,t) + (\mathscr{K}_3\bm{z})(t)\bigr] \\
& \qquad \qquad \qquad \qquad \qquad \quad + (\mathscr{K}_4\bm{z})(t) + \bm{h}_2 \Bigl(\bm{v}\bigl(\bm{x}(t),\bm{\delta}(t)\bigr)\Bigr), \quad (\xi,t) \in (0,1)\times(0,T), 
\end{split}\\
& \bm{z}(0,t) = \bm{0}, \quad t \in (0,T), \label{eq:BCoriginal}
\end{align}
\end{subequations}
where $\mathbb{R}^2 \ni \bm{x}(t) \triangleq [v_y(t) \; r(t)]^{\mathrm{T}}$ denotes the lumped state vector, $\mathbb{R}^2 \ni \bm{z}(\xi,t) \triangleq [z_1(\xi,t)\; z_2(\xi,t)]^{\mathrm{T}}$ represents the distributed state vector, $\mathbb{R}^2 \ni \bm{\delta}(t) \triangleq [\delta_1(t) \; \delta_2(t)]^{\mathrm{T}}$ indicates the input, and the rigid relative velocity $\bm{v} \in C^1(\mathbb{R}^{4 };\mathbb{R}^{2})$ reads
\begin{align}\label{eq:relVel}
\bm{v}(\bm{x},\bm{\delta}) = \mathbf{A}_2\bm{x} + \mathbf{G}_2\bm{\delta}.
\end{align} 
In~\eqref{eq:originalSystems} and~\eqref{eq:relVel}, the diagonal matrix $\mathbf{GL}_{2 }(\mathbb{R})\cap \mathbf{Sym}_{2 }(\mathbb{R})  \ni \mathbf{\Lambda}\succ \mathbf{0}$ collects the transport velocities, $\mathbf{\Sigma} \in C^0(\mathbb{R}^{2 };\mathbf{M}_{2 }(\mathbb{R}))$ represents the nonlinear source matrix, the matrices $\mathbf{A}_1, \mathbf{A}_2, \mathbf{G}_1, \mathbf{G}_2\in  \mathbf{M}_{2}(\mathbb{R})$ have constant coefficients, and $\bm{h}_1, \bm{h}_2 \in C^0(\mathbb{R}^{2 };\mathbb{R}^{2 })$ are vector-valued functions.
Finally, the operators $(\mathscr{K}_1\bm{\zeta})$, $(\mathscr{K}_2\bm{\zeta})$, $(\mathscr{K}_3\bm{\zeta})$, $(\mathscr{K}_4\bm{\zeta})$ satisfy $\mathscr{K}_1, \mathscr{K}_4 \in\mathscr{L}(H^1((0,1);\mathbb{R}^{2 });\mathbb{R}^{2 })$ and $\mathscr{K}_2, \mathscr{K}_3 \in \mathscr{L}(L^2((0,1);\mathbb{R}^{2 });\mathbb{R}^{2 })$, and are given by
\begin{subequations}\label{eq;operatorsKs}
\begin{align}
(\mathscr{K}_1\bm{\zeta}) &\triangleq   \int_0^1 \mathbf{K}_1(\xi) \bm{\zeta}(\xi) \dif \xi + \mathbf{K}_2\bm{\zeta}(1), \label{eq:operatorK_1}\\
(\mathscr{K}_2\bm{\zeta}) &\triangleq   \int_0^1\mathbf{K}_3(\xi)\bm{\zeta}(\xi) \dif \xi, \label{eq:operatorK_2} \\
(\mathscr{K}_3\bm{\zeta}) &\triangleq   \int_0^1 \mathbf{K}_4(\xi) \bm{\zeta}(\xi) \dif \xi, \label{eq:operatorK_3}\\
(\mathscr{K}_4\bm{\zeta}) &\triangleq   \int_0^1 \mathbf{K}_5(\xi) \bm{\zeta}(\xi) \dif \xi+ \mathbf{K}_6\bm{\zeta}(1), \label{eq:operatorK_4} 
\end{align}
\end{subequations}
with $\mathbf{K}_1,\mathbf{K}_3, \mathbf{K}_4, \mathbf{K}_5 \in C^0([0,1];\mathbf{M}_{2 }(\mathbb{R}))$, and $\mathbf{K}_2, \mathbf{K}_6 \in \mathbf{M}_{2 }(\mathbb{R})$. 

The semilinear representation~\eqref{eq:originalSystems} provides a compact and powerful framework to investigate the dynamics of the single-track models developed in this paper. Indeed, as shown in the next Sects.~\ref{sect:semiRigid} and~\ref{sect:semiFlex}, adequate specification of the matrices figuring in Eqs.~\eqref{eq:originalSystems},~\eqref{eq:relVel}, and~\eqref{eq;operatorsKs} may accommodate both the formulations with rigid and compliant tyre carcass.

\subsubsection{Semilinear single-track models with rigid carcass}\label{sect:semiRigid}
The semilinear single-track model with rigid carcass admits a state-space representation in the form described by Eqs.~\eqref{eq:originalSystems},~\eqref{eq:relVel}, and~\eqref{eq;operatorsKs}, with
\begin{align}\label{eq:ssSiRig}
\mathbf{A}_1 & \triangleq \begin{bmatrix} 0 & -v_x \\ 0 & 0 \end{bmatrix}, \quad \mathbf{G}_1 \triangleq  -\begin{bmatrix} \dfrac{1}{m} &  \dfrac{1}{m} \\
 \dfrac{l_1}{I_z}&  -\dfrac{l_2}{I_z}\end{bmatrix}, \nonumber \\ 
\mathbf{K}_1(\xi) & \triangleq \begin{bmatrix} F_{z1}\biggl(\sigma_{0,1}\bar{p}_1(\xi) + \chi_2\dfrac{v_x\sigma_{1,1}}{L_1}\dod{\bar{p}_1(\xi)}{\xi} \biggr) & 0 \\ 0 & F_{z2}\biggl(\sigma_{0,2}\bar{p}_2(\xi) + \chi_2\dfrac{v_x\sigma_{1,2}}{L_2}\dod{\bar{p}_2(\xi)}{\xi} \biggr)\end{bmatrix}, \nonumber \\
\mathbf{K}_2 & \triangleq -\begin{bmatrix} \chi_2F_{z1}\dfrac{v_x\sigma_{1,1}}{L_1}\bar{p}_1(1) & 0 \\ 0 & \chi_2F_{z2}\dfrac{v_x\sigma_{1,2}}{L_2}\bar{p}_2(1)  \end{bmatrix}, \quad \mathbf{K}_3(\xi) \triangleq \begin{bmatrix} F_{z1}\sigma_{1,1}\bar{p}_1(\xi) & 0 \\ 0 &F_{z2}\sigma_{1,2}\bar{p}_2(\xi)   \end{bmatrix}, \nonumber \\
\bm{h}_1(\bm{v}) & \triangleq 2\begin{bmatrix} F_{z1}\biggl(\sigma_{1,1}\dfrac{\mu_1(v_1)}{g_1(v_1;\chi_1)} + \sigma_{2,1}\biggr)v_1  \\ F_{z2}\biggl(\sigma_{1,2}\dfrac{\mu_2(v_2)}{g_2(v_2;\chi_1)} + \sigma_{2,2}\biggr)v_2 \end{bmatrix}, \quad  \mathbf{\Lambda} \triangleq \begin{bmatrix} \dfrac{v_x}{L_1} & 0 \\ 0 & \dfrac{v_x}{L_2}\end{bmatrix}, \nonumber \\
\mathbf{\Sigma}(\bm{v}) & \triangleq -\begin{bmatrix} \dfrac{\sigma_{0,1}\abs{v_{1}}_\varepsilon}{g_{1}(v_{1};\chi_1)} & 0 \\
0 & \dfrac{\sigma_{0,2}\abs{v_{2}}_\varepsilon}{g_{2}(v_{2};\chi_1)}  \end{bmatrix}, \quad
\bm{h}_2(\bm{v})  \triangleq 2\begin{bmatrix} \dfrac{\mu_1(v_1)}{g_1(v_1;\chi_1)}v_1 \\ \dfrac{\mu_2(v_2)}{g_2(v_2;\chi_1)}v_2 \end{bmatrix}, \nonumber \\
\end{align}
$\mathbf{K}_4 = \mathbf{K}_5 = \mathbf{K}_6 = \mathbf{0}$, and the matrices in Eq.~\eqref{eq:relVel} reading
\begin{align}\label{eq:A2G2}
\mathbf{A}_2 & \triangleq \begin{bmatrix}1 & l_1 \\ 1 & -l_1 \end{bmatrix}, && \mathbf{G}_2  \triangleq -v_x\begin{bmatrix} 1 & 0 \\ 0 & \chi_3\end{bmatrix}.
\end{align}

\subsubsection{Semilinear single-track models with flexible carcass}\label{sect:semiFlex}
The semilinear single-track model with flexible carcass may be put compactly in the form described by Eqs.~\eqref{eq:originalSystems},~\eqref{eq:relVel}, and~\eqref{eq;operatorsKs}, with $\mathbf{A}_1$, $\mathbf{G}_1$, and $\mathbf{\Lambda}$ as in~\eqref{eq:ssSiRig}, $\mathbf{K}_2 = \mathbf{K}_3 = \mathbf{0}$, $\bm{h}_1(\bm{v}) = \bm{0}$, $\mathbf{A}_2$ and $\mathbf{G}_2$ according to~\eqref{eq:A2G2}, and
\begin{align}\label{eq:matFlex}
\mathbf{K}_1(\xi) & \triangleq \begin{bmatrix} F_{z1}\sigma_{0,1}\bar{p}_1(\xi) & 0 \\ 0 &  F_{z2}\sigma_{0,2}\bar{p}_2(\xi) \end{bmatrix}, && \mathbf{\Sigma}(\bm{v})  \triangleq -\begin{bmatrix} \dfrac{\sigma_{0,1}\abs{v_{1}}_\varepsilon}{\mu_{1}(v_{1})} & 0 \\
0 & \dfrac{\sigma_{0,2}\abs{v_{2}}_\varepsilon}{\mu_{2}(v_{2})}  \end{bmatrix}, \nonumber \\
 \mathbf{K}_4(\xi) & \triangleq -\begin{bmatrix}\psi_{1}\bar{p}_{1}(\xi) & 0  \\ 0 &\psi_{2}\bar{p}_{2}(\xi)  \end{bmatrix},  && \mathbf{K}_5(\xi)  \triangleq -v_x\begin{bmatrix}\dfrac{\psi_{1}}{L_{1}}\dod{\bar{p}_{1}(\xi)}{\xi} & 0 \\ 0 &\dfrac{\psi_{2}}{L_{2}}\dod{\bar{p}_{2}(\xi)}{\xi}   \end{bmatrix},  \nonumber\\
\mathbf{K}_6 & \triangleq v_x\begin{bmatrix}\dfrac{\psi_{1}\bar{p}_{1}(1)}{L_{1}} & 0  \\ 0 &\dfrac{\psi_{2}\bar{p}_{2}(1)}{L_{2}} \end{bmatrix}, && \bm{h}_2(\bm{v}) \triangleq 2\begin{bmatrix} \phi_1 & 0 \\ 0 & \phi_2 \end{bmatrix}\bm{v}.
\end{align}
Having shown that the ODE-PDE representation~\eqref{eq:originalSystems} effectively covers both the model variants considered in the manuscript, the next Sect.~\ref{sect:well} proceeds to study its well-posedness. Whereas some technical details are confined to Appendix~\ref{app:theorems}, it is essential to clarify that the following analysis does not aim at full generality, striving instead to achieve an armonious balance between rigour and physical insight. Indeed, beyond their mere theoretical relevance, the results advocated in Sect.~\ref{sect:well} permit drawing important considerations about the advantages and drawbacks of the different friction models reviewed previously in Sect.~\ref{Sect:friction}.

\subsection{Well-posedness}\label{sect:well}
For the semilinear ODE-PDE interconnection~\eqref{eq:originalSystems}, well-posedness may be proved locally within different functional settings. In particular, this paper focuses on the notions of \emph{mild} and \emph{classical solutions}, whose existence and uniqueness may be established by resorting to classic semigroup arguments. To this end, the Hilbert spaces $\mathcal{X}\triangleq \mathbb{R}^{2}\times L^2((0,1);\mathbb{R}^{2 })$ and $\mathcal{Y}\triangleq \mathbb{R}^{2}\times H^1((0,1);\mathbb{R}^{2 })$, equipped respectively with norms $\norm{(\bm{y}, \bm{\zeta}(\cdot))}_{\mathcal{X}}^2 \triangleq \norm{\bm{y}}_{2}^2 + \norm{\bm{\zeta}(\cdot)}_{L^2((0,1);\mathbb{R}^{2 })}^2$ and $\norm{(\bm{y}, \bm{\zeta}(\cdot))}_{\mathcal{Y}}^2 \triangleq \norm{\bm{y}}_{2}^2 + \norm{\bm{\zeta}(\cdot)}_{H^1((0,1);\mathbb{R}^{2 })}^2$, are considered. Accordingly, Theorems~\ref{thm:mild} and~\ref{thm:class} deliver the first theoretical results of the manuscript.

\begin{theorem}[Local existence and uniqueness of mild solutions]\label{thm:mild}
Suppose that $\mathbf{\Sigma} \in C^0(\mathbb{R}^{2 };\mathbf{M}_{2 }(\mathbb{R}))$ and $\bm{h}_1, \bm{h}_2\in C^0(\mathbb{R}^{2 };\mathbb{R}^{2 })$ are locally Lipschitz continuous, and $\bm{\delta} \in C^0([0,T];\mathbb{R}^{2})$. Then, for all initial conditions (ICs) $(\bm{x}_0,\bm{z}_0) \triangleq (\bm{x}(0),\bm{z}(\cdot,0)) \in \mathcal{X}$, there exists $ t\ped{max} \leq \infty$ such that the ODE-PDE system~\eqref{eq:originalSystems} admits a unique \emph{mild solution} $(\bm{x},\bm{z}) \in C^0([0,t\ped{max});\mathcal{X})$.
\begin{proof}
See Appendix~\ref{app:WellP}.
\end{proof}
\end{theorem}

\begin{theorem}[Local existence and uniqueness of classical solutions]\label{thm:class}
Suppose that $\mathbf{\Sigma} \in C^1(\mathbb{R}^{2 };\mathbf{M}_{2 }(\mathbb{R}))$, $\bm{h}_1, \bm{h}_2\in C^1(\mathbb{R}^{2 };\mathbb{R}^{2 })$, and $\bm{\delta} \in C^1([0,T];\mathbb{R}^{2})$. Then, for all ICs $(\bm{x}_0,\bm{z}_0) \in \mathcal{Y}$ satisfying the BC~\eqref{eq:BCoriginal}, there exists $ t\ped{max} \leq \infty$ such that the ODE-PDE system~\eqref{eq:originalSystems} admits a unique \emph{classical solution} $(\bm{x},\bm{z}) \in C^1([0,t\ped{max});\mathcal{X}) \cap C^0([0,t\ped{max}); \mathcal{Y})$ satisfying the BC~\eqref{eq:BCoriginal}.
\begin{proof}
See Appendix~\ref{app:WellP}.
\end{proof}
\end{theorem}

Unfortunately, due to the semilinear nature of the interconnection~\eqref{eq:originalSystems}, Theorems~\ref{thm:mild} and~\ref{thm:class} assert the existence of solutions defined only locally in time. Nonetheless, under the same assumptions as Theorem~\ref{thm:class}, it is possible to enounce global well-posedness by imposing stricter requirements on the functions $\mathbf{\Sigma}(\cdot)$, $\bm{h}_1(\cdot)$, and $\bm{h}_2(\cdot)$ appearing in Eqs.~\eqref{eq:originalSystems}. In this context, a stronger result, concerned with mild solutions, is formalised in Theorem~\ref{thm:global} below.

\begin{theorem}[Global existence and uniqueness of mild solutions]\label{thm:global}
Suppose that $\mathbf{\Sigma} \in C^1(\mathbb{R}^{2 };\mathbf{M}_{2 }(\mathbb{R}))$, $\bm{\delta} \in C^1(\mathbb{R}_{\geq 0};\mathbb{R}^{2}) \cap L^\infty(\mathbb{R}_{\geq 0};\mathbb{R}^{2 })$, and $\bm{h}_1, \bm{h}_2\in C^1(\mathbb{R}^{2 };\mathbb{R}^{2 })$ satisfy a linear growth condition, that is, there exist $L_{h_1}, L_{h_2}\in \mathbb{R}_{\geq 0}$ and $b_1,b_2 \in \mathbb{R}_{\geq 0}$ such that
\begin{subequations}\label{eq:hCond}
\begin{align}
\norm{\bm{h}_1(\bm{v})}_2 & \leq L_{h_1} \norm{\bm{v}}_2 + b_1, \\
\norm{\bm{h}_2(\bm{v})}_2 & \leq L_{h_2} \norm{\bm{v}}_2 + b_2.
\end{align}
\end{subequations}
Moreover, assume that one of the following Hypotheses holds:
\begin{enumerate}[(H.1)]
\item $\mathbf{K}_2 =\mathbf{K}_3 = \mathbf{0}$, and there exists a diagonal matrix-valued function $C^0([0,1];\mathbf{Sym}_{2}(\mathbb{R})) \ni \mathbf{P} \triangleq \diag\{P_1, P_2\}$, with $\mathbf{P}(\xi) \succ \mathbf{0}$, such that, for every $(\bm{y},\bm{\zeta}) \in \mathbb{R}^{2 }\times L^2((0,1);\mathbb{R}^{2 })$,
\begin{align}\label{eq:ineqP}
\int_0^1 \bm{\zeta}^{\mathrm{T}}(\xi)\mathbf{P}(\xi) \mathbf{\Sigma}(\bm{y})\bigl[\bm{\zeta}(\xi) + (\mathscr{K}_3\bm{\zeta})\bigr] \dif \xi \leq 0.
\end{align}\label{th:ext.1}
\item For all $\bm{y} \in \mathbb{R}^{2 }$, the matrix $\mathbf{\Sigma}(\bm{y}) \in \mathbf{M}_{2}(\mathbb{R})$ satisfies
\begin{align}
\norm{\mathbf{\Sigma}(\bm{y})} \leq M_{\mathbf{\Sigma}}, 
\end{align}
for some $M_{\mathbf{\Sigma}} \in \mathbb{R}_{\geq 0}$. \label{th:ext.2}
\end{enumerate}
Then, for all ICs $(\bm{x}_0,\bm{z}_0)\in \mathcal{X}$, the ODE-PDE system~\eqref{eq:originalSystems} admits a unique global mild solution $(\bm{x},\bm{z}) \in C^0(\mathbb{R}_{\geq 0};\mathcal{X})$.
\begin{proof}
See Appendix~\ref{app:thmGlobal}
\end{proof}
\end{theorem}
For simplicity, Theorem~\ref{thm:global} is stated above owing to the assumption of smooth matrices and inputs, as also required by Theorem~\ref{thm:class}; the extension to non-differentiable coefficients and inputs would involve subtle technicalities that could obscure the main message of the paper, and is therefore not attempted here. 
Instead, it is worth commenting on the conditions under which the inequalities~\eqref{eq:hCond}, as well as Hypotheses~\ref{th:ext.1} and~\ref{th:ext.2}, hold. Concerning the bounds in Eqs.~\eqref{eq:hCond}, inspection of~\eqref{eq:ssSiRig} and~\eqref{eq:matFlex} immediately reveals that they are identically satisfied by the Dahl, LuGre, and FrBD friction models. Considering first the semilinear single-track formulation with rigid tyre carcass described in Sect.~\ref{sect:semiRigid}, Hypothesis~\ref{th:ext.2} is only verified by the FrBD model (with $\chi_1 =1$ in Eq.~\eqref{eq:ssSiRig}). However, adopting a distributed Dahl model ($\sigma_{1,i} = \sigma_{2,i} = 0$) ensures the fulfillment of Hypothesis~\ref{th:ext.1} for any combination of physical parameters; the same condition is verified if a linear viscous damping term is incorporated into the expressions for the friction coefficients $\mu_1(\cdot)$ and $\mu_2(\cdot)$, as done by the FrBD model. Moving instead to the variant with a flexible carcass, as detailed in Sect.~\ref{sect:semiFlex}, it is clear that Hypothesis~\ref{th:ext.2} is again satisfied whenever $\mu_1(\cdot)$ and $\mu_2(\cdot)$ grow linearly in their arguments, consistently with the modelling approach of the FrBD formulation. Finally, sufficient conditions for Hypothesis~\ref{th:ext.1} to hold are enounced by Proposition~\ref{lemma:P} below.

\begin{proposition}\label{lemma:P}
Consider the matrices $\mathbf{\Sigma}(\bm{y})$ and $\mathbf{K}_4(\xi)$ as in Eq.~\eqref{eq:matFlex}, and suppose that $\bar{p}_i \in C^1([0,1];\mathbb{R}_{>0})$, $i \in \{1,2\}$, and $\psi_i \in (0,1)$ satisfies
\begin{align}\label{eq:psiComnd}
\psi_i \norm{\bar{p}_i(\cdot)}_\infty \leq 1, \quad i \in \{1,2\}.
\end{align}
Then, the matrix $\mathbf{P}(\xi)$ in Hypothesis~\ref{th:ext.1} may be selected as $\mathbf{P}(\xi) = \diag\{\bar{p}_1(\xi), \bar{p}_2(\xi)\}$.
\begin{proof}
First, it should be observed that $\bar{p}_i\in C^1([0,1];\mathbb{R}_{>0})$ implies that $\mathbf{P}(\xi) \triangleq \diag\{\bar{p}_1(\xi),\bar{p}_2(\xi)\}$ is positive definite, i.e., $\mathbf{P}(\xi) \succ \mathbf{0}$, with $\mathbf{P} \in C^1([0,1];\mathbf{Sym}_{2}(\mathbb{R}))$. Moreover, if Eq.~\eqref{eq:psiComnd} holds, computing~\eqref{eq:ineqP} gives, after simple manipulations,
\begin{align}\label{eq:ineqP2}
\begin{split}
\int_0^1 \bm{\zeta}^{\mathrm{T}}(\xi)\mathbf{P}(\xi) \mathbf{\Sigma}(\bm{y})\bigl[\bm{\zeta}(\xi) + (\mathscr{K}_3\bm{\zeta})\bigr] \dif \xi &\leq -\dfrac{\sigma_{0,1}\abs{y_1}_\varepsilon}{\mu_1(y_1)}\int_0^1 \bar{p}_1(\xi)\bigl(1-\psi_1\bar{p}_1(\xi)\bigr)\zeta_1^2(\xi) \dif \xi \\
& \quad -\dfrac{\sigma_{0,2}\abs{y_2}_\varepsilon}{\mu_2(y_2)}\int_0^1 \bar{p}_2(\xi)\bigl(1-\psi_2\bar{p}_2(\xi)\bigr)\zeta_2^2(\xi) \dif \xi \leq 0,
\end{split}
\end{align}
which concludes the proof.
\end{proof}
\end{proposition}
The conditions presented in Proposition~\ref{lemma:P} require the specification of a particular pressure distribution within the tyre's contact patch, such as a constant or exponentially decreasing profile (the compactly supported parabolic one being clearly not suitable). In this context, the following Remark~\ref{remark:1} highlights some key considerations.
\begin{remark}\label{remark:1}
Since $\psi_i \in (0,1)$, $i \in \{1,2\}$, the inequality~\eqref{eq:psiComnd} is automatically satisfied when $\norm{\bar{p}_i(\cdot)}_\infty = 1$, as it happens for the case of a constant pressure distribution. In contrast, for an exponentially decreasing profile, the maximum pressure may be expressed as $\norm{\bar{p}_i(\cdot)}_\infty \equiv \bar{p}_i(0) = \frac{a_i}{1-\eu^{-a_i}} > 1$. If the parameter $a_i$ is allowed to be arbitrarily small, then Eq.~\eqref{eq:psiComnd} is virtually satisfied for any value of $\psi_i \in (0,1)$, $i \in \{1,2\}$. In practice, the constant pressure profile is commonly used in the literature and typically yields accurate results concerning the computation of the tyre forces. However, exponentially decreasing profiles offer the additional benefit of guaranteeing strict dissipativity of the friction model. This property is particularly advantageous when it comes to the synthesis of dissipativity-based controllers and estimators, making exponential pressure distributions highly desirable in control-oriented applications.
\end{remark}
In essence, the distributed Dahl and FrBD models are the only formulations that satisfy Hypotheses~\ref{th:ext.1} and~\ref{th:ext.2} across all relevant cases of practical interest. This is intuitively expected, as both models are grounded in physically consistent principles, as explained in \cite{FrBD}. In contrast, the LuGre model is derived from a more heuristic rationale, which lacks a fully rigorous physical justification.

\section{Linear and linearised models}\label{sect:linearLinearised}
The present section is dedicated to the analysis of a linearised version of the ODE-PDE system~\eqref{eq:originalSystems}. Indeed, besides providing important insights about the local behaviour of the semilinear single-track models derived in Sect.~\ref{sect:SemilienarModels}, linearisation may additionally facilitate the synthesis of controllers and observers.
In this context, the next Sect.~\ref{sect:LinearEqui} introduces the linearised equations of the model, whereas Sect.~\ref{sect.Stabiliy} derives necessary and sufficient conditions for uniform exponential stability, and then moves to study the frequency response of the linearised system.

\subsection{Equilibria and linearisation}\label{sect:LinearEqui}
To proceed with linearisation of the ODE-PDE system~\eqref{eq:originalSystems}, its equilibria need first to be determined, which is accomplished in Sect.~\ref{sect:equi}. The complete set of equations for the linearised single-track models is then provided in Sect.~\ref{sect:linearisation}, where some technical results are also enounced.

\subsubsection{System's equilibria}\label{sect:equi}
The first step in the linearisation of Eq.~\eqref{eq:originalSystems} consists in the determination the equilibria $(\bm{x}^\star, \bm{z}^\star) \in \mathcal{Y}$ associated with the constant input $\bm{\delta}(t) = \bm{\delta}^\star \in \mathbb{R}^2$. Since the PDEs governing the bristle dynamics, Eqs.~\eqref{eq:FrBD} and~\eqref{eq:ooCarcass}, are highly nonlinear, an analytical treatment is, in general, not possible, and numerical approaches are required. However, some preliminary insights may be gained by considering the steady-state form of Eq.~\eqref{eq:bycicle}:
\begin{subequations}
\begin{align}
r^\star  +\dfrac{1}{m v_x}\Bigl(F_{y1}\bigl(v_1(\bm{x}^\star,\delta_1^\star)\bigr) + F_{y2}\bigl(v_2(\bm{x}^\star, \delta_2^\star)\bigr)\Bigr) & = 0, \\
l_1F_{y1}\bigl(v_1(\bm{x}^\star,\delta_1^\star)\bigr) - l_2F_{y2}\bigl(v_2(\bm{x}^\star, \delta_2^\star)\bigr) & = 0, \\
\end{align}
\end{subequations}
where the stationary rigid relative velocities read
\begin{subequations}
\begin{align}
v_1(\bm{x}^\star,\delta_1^\star) & = v_y^\star + l_1r^\star-v_x\delta_1^\star, \\
v_2(\bm{x}^\star, \delta_2^\star) & = v_y^\star - l_2r^\star-\chi_3v_x\delta_2^\star.
\end{align}
\end{subequations}
and the steady-state forces $F_{yi}(v_i(\bm{x}^\star,\delta_i^\star))$, $i \in \{1,2\}$, are functions of the relative rigid velocities and bristle deformations via the integral relationships in Eq.~\eqref{eq:Forces1}. The stationary solution to Eqs.~\eqref{eq:FrBD} and~\eqref{eq:ooCarcass} reads, in turn,
\begin{align}\label{eq:CD22}
z_i^\star (\xi) = 2\sgn_\varepsilon\bigl(v_i(\bm{x}^\star,\delta_i^\star) \bigr)\dfrac{\mu_i\bigl(v_i(\bm{x}^\star,\delta_i^\star) \bigr)}{\sigma_{0,i}}\Biggl[1-\exp\biggl( -\dfrac{L_i\sigma_{0,i}\abs{v_i(\bm{x}^\star,\delta_i^\star) }_\varepsilon}{v_x g_i\bigl(v_i(\bm{x}^\star,\delta_i^\star);\chi_1)}\xi\biggr)\Biggr], \quad \xi \in [0,1],
\end{align}
for $i \in \{1,2\}$, with $\chi_1 = 0$ if the formulation with flexible tyre carcass is employed. Importantly, inspection of the above Eq.~\eqref{eq:CD22} reveals that $z_i^\star(\xi) = 0$ if and only if $v_i(\bm{x}^\star,\delta_i^\star) = 0$, $i \in \{1,2\}$, which, in conjunction with $r^\star = 0$, implies
\begin{subequations}
\begin{align}
 v_y^\star -v_x\delta_1^\star = 0, \\
v_y^\star -\chi_3v_x\delta_2^\star = 0,
\end{align}
\end{subequations}
and thus $\delta_1^\star = \chi_3\delta_2^\star = v_y^\star/v_x$. This means that there are infinite equilibria of the type $(\bm{x}^\star, \bm{z}^\star(\xi), \bm{\delta}^\star) = (v_y^\star, 0, \bm{0},\bm{\delta}^\star)$, but the condition $\bm{v}(\bm{x}^\star,\bm{\delta}) = \bm{0}$ necessarily requires $\bm{z}^\star(\xi) = \bm{0}$. This observation plays a vital role in the linearisation of Eq.~\eqref{eq:originalSystems}, which is addressed next in Sect.~\ref{sect:linearisation}. 

\subsubsection{Linearisation}\label{sect:linearisation}
With the equilibria $(\bm{x}^\star, \bm{z}^\star, \bm{\delta}^\star) \in \mathcal{Y}\times \mathbb{R}^2$ of the ODE-PDE system~\eqref{eq:originalSystems} determined as in Sect.~\ref{sect:equi}, it is possible to proceed with its linearisation. However, since the functions $\mathbf{\Sigma}(\cdot)$, $\bm{h}_1(\cdot)$, and $\bm{h}_2(\cdot)$ have not been assumed to be differentiable everywhere, it is first necessary to formulate certain hypotheses concerning their behaviour. In particular, in the remainder of this section, it is assumed that $\mathbf{\Sigma} \in C^1(\mathbb{R}^2\setminus \{\mathbf{0}\}; \mathbf{M}_2(\mathbb{R}))$, and $\bm{h}_1,\bm{h}_2 \in C^1(\mathbb{R}^2\setminus \{\mathbf{0}\}; \mathbb{R}^2)$, and additionally that $\lim_{\bm{v}\to \bm{0}}\nabla_{\bm{v}}\bm{h}_1(\bm{v})$ and $\lim_{\bm{v}\to \bm{0}} \nabla_{\bm{v}}\bm{h}_2(\bm{v})$ exist and are finite. Such conditions are always verified in practice. Owing to these premises, the ODE-PDE~\eqref{eq:originalSystems} interconnection may be linearised around an equilibrium $(\bm{x}^\star, \bm{z}^\star, \bm{\delta}^\star) \in \mathcal{Y}\times\mathbb{R}^2$. Introducing the perturbation variables $\mathbb{R}^2 \ni \tilde{\bm{x}}(t) \triangleq \bm{x}(t)-\bm{x}^\star$, $\mathbb{R}^2 \ni \tilde{\bm{z}}(\xi,t) \triangleq \bm{z}(\xi,t)-\bm{z}^\star(\xi)$, and $\mathbb{R}^2 \ni \tilde{\bm{\delta}}(t) \triangleq \bm{\delta}(t)-\bm{\delta}^\star$, the following linearised ODE-PDE system is obtained:
\begin{subequations}\label{eq:originalSystemsLin}
\begin{align}
\begin{split}
& \dot{\tilde{\bm{x}}}(t) = \tilde{\mathbf{A}}_1\bigl(\bm{x}^\star,\bm{\delta}^\star,\bm{z}^\star\bigr)\tilde{\bm{x}}(t) + \mathbf{G}_1\Bigl[(\mathscr{K}_1\tilde{\bm{z}})(t) + \mathbf{\Sigma}\bigl(\bm{v}(\bm{x}^\star,\bm{\delta}^\star)\bigr)(\mathscr{K}_2\tilde{\bm{z}})(t)+\tilde{\mathbf{B}}_1\bigl(\bm{x}^\star,\bm{\delta}^\star,\bm{z}^\star\bigr)\tilde{\bm{\delta}}(t)\Bigr], \quad t \in (0,T), \end{split}\\
\begin{split}
& \dpd{\tilde{\bm{z}}(\xi,t)}{t} + \mathbf{\Lambda}\dpd{\tilde{\bm{z}}(\xi,t)}{\xi} = \mathbf{\Sigma}\bigl(\bm{v}(\bm{x}^\star,\bm{\delta}^\star)\bigr)\bigl[\tilde{\bm{z}}(\xi,t) + (\mathscr{K}_3\tilde{\bm{z}})(t)\bigr]+ (\mathscr{K}_4\tilde{\bm{z}})(t) \\
& \qquad \qquad \qquad \qquad \qquad \quad  +\tilde{\mathbf{A}}_2\bigl(\bm{x}^\star,\bm{\delta}^\star,\bm{z}^\star(\xi)\bigr)\tilde{\bm{x}}(t) + \tilde{\mathbf{B}}_2 \bigl(\bm{x}^\star,\bm{\delta}^\star,\bm{z}^\star(\xi)\bigr) \tilde{\bm{\delta}}(t), \quad (\xi,t) \in (0,1)\times(0,T), 
\end{split}\\
& \tilde{\bm{z}}(0,t) = \bm{0}, \quad t \in (0,T), \label{eq:BCoriginalLin}
\end{align}
\end{subequations}
where
\begin{subequations}\label{eq:matricesTilde}
\begin{align}
\tilde{\mathbf{A}}_1\bigl(\bm{x}^\star,\bm{\delta}^\star,\bm{z}^\star\bigr) & \triangleq \mathbf{A}_1 + \mathbf{G}_1 \mathbf{H}_1\bigl(\bm{x}^\star,\bm{\delta}^\star,\bm{z}^\star\bigr)\mathbf{A}_2, \\
\tilde{\mathbf{A}}_2\bigl(\bm{x}^\star,\bm{\delta}^\star,\bm{z}^\star(\xi)\bigr)  & \triangleq \mathbf{H}_2\bigl(\bm{x}^\star,\bm{\delta}^\star,\bm{z}^\star(\xi)\bigr)\mathbf{A}_2, \\
\tilde{\mathbf{B}}_1\bigl(\bm{x}^\star,\bm{\delta}^\star,\bm{z}^\star\bigr) & \triangleq \mathbf{H}_1\bigl(\bm{x}^\star,\bm{\delta}^\star,\bm{z}^\star\bigr)\mathbf{G}_2, \\
\tilde{\mathbf{B}}_2 \bigl(\bm{x}^\star,\bm{\delta}^\star,\bm{z}^\star(\xi)\bigr) & \triangleq \mathbf{H}_2\bigl(\bm{x}^\star,\bm{\delta}^\star,\bm{z}^\star(\xi)\bigr)\mathbf{G}_2,
\end{align}
\end{subequations}
with
\begin{subequations}\label{eq:matricesH}
\begin{align}
 \mathbf{H}_1\bigl(\bm{x}^\star,\bm{\delta}^\star,\bm{z}^\star\bigr) & \triangleq \eval{\nabla_{\bm{v}}\mathbf{\Sigma}(\bm{v})^{\mathrm{T}}}_{\bm{v} = \bm{v}(\bm{x}^\star, \bm{\delta}^\star)}\diag\{(\mathscr{K}_2\bm{z}^\star)\} + \eval{\nabla_{\bm{v}}\bm{h}_1(\bm{v})^{\mathrm{T}}}_{\bm{v} = \bm{v}(\bm{x}^\star, \bm{\delta}^\star)}, \\
\mathbf{H}_2\bigl(\bm{x}^\star,\bm{\delta}^\star,\bm{z}^\star(\xi)\bigr) & \triangleq \eval{\nabla_{\bm{v}}\mathbf{\Sigma}(\bm{v})^{\mathrm{T}}}_{\bm{v} = \bm{v}(\bm{x}^\star, \bm{\delta}^\star)}\diag\{\bm{z}^\star(\xi)\} + \eval{\nabla_{\bm{v}}\bm{h}_2(\bm{v})^{\mathrm{T}}}_{\bm{v} = \bm{v}(\bm{x}^\star, \bm{\delta}^\star)}.
\end{align}
\end{subequations}
It is essential to remark that, based on the considerations drawn in Sect.~\ref{sect:equi} relatively to the structure of the equilibria $(\bm{x}^\star, \bm{z}^\star(\xi), \bm{\delta}^\star) = (v_y^\star, 0, \bm{0}, \bm{\delta}^\star)$, and the assumptions postulated on the functions $\bm{h}_1(\cdot)$ and $\bm{h}_2(\cdot)$, the quantities figuring in Eq.~\eqref{eq:matricesH} are well-defined. Indeed, since $\bm{v} = \bm{0}$ identically implies $\bm{z}^\star(\xi) = \bm{0}$, the first terms in Eq.~\eqref{eq:matricesH} must vanish at $\bm{v} = \bm{0}$ despite the possible non-differentiability of $\mathbf{\Sigma}(\cdot)$, whereas the second ones exist and are bounded. Therefore, the linearised single-track model described by Eqs.~\eqref{eq:originalSystemsLin},~\eqref{eq:matricesTilde}, and~\eqref{eq:matricesH} can adequately describe the lateral motion of the vehicle also in the limiting case of zero sideslip angles, or equivalently, rigid relative velocities.

For completeness, the global well-posedness of the linearised ODE-PDE system~\eqref{eq:originalSystemsLin} is asserted by Theorem~\ref{thm:Ulin} below.
\begin{theorem}[Global existence and uniqueness of solutions]\label{thm:Ulin}
The ODE-PDE system~\eqref{eq:originalSystemsLin} admits a unique mild solution $(\tilde{\bm{x}},\tilde{\bm{z}}) \in C^0([0,T];\mathcal{X})$ for all ICs $(\tilde{\bm{x}}_0,\tilde{\bm{z}}_0) \triangleq (\tilde{\bm{x}}(0), \tilde{\bm{z}}(\cdot,0)) \in \mathcal{X}$ and inputs $\tilde{\bm{\delta}} \in L^p((0,T);\mathbb{R}^2)$, $p\geq 1$. If, in addition, $(\tilde{\bm{x}}_0,\tilde{\bm{z}}_0) \in \mathcal{Y}$ satisfies the BC~\eqref{eq:BCoriginalLin}, and $\tilde{\bm{\delta}} \in C^1([0,T];\mathbb{R}^2)$, the solution is classical, that is, $(\tilde{\bm{x}}, \tilde{\bm{z}}) \in C^1([0,T];\mathcal{X})\cap C^0([0,T];\mathcal{Y})$ and satisfies the BC~\eqref{eq:BCoriginalLin}.
\begin{proof}
See Appendix~\ref{app:ProofLin}.
\end{proof}
\end{theorem}

\subsection{Stability and transfer function}\label{sect.Stabiliy}
The present section is devoted to investigating the stability of the linearised ODE-PDE system~\eqref{eq:originalSystemsLin}. Specifically, the formal conditions to deduce (exponential) stability are stated in Sect.~\ref{app:Spectral}, where an exhaustive spectral analysis is conducted; Sect.~\ref{app:output} derives instead an explicit expression for the transfer function of the system, where the output is defined in terms of lumped states.
The theoretical findings of both Sects.~\ref{app:Spectral} and~\ref{app:output} are accompanied by illustrative numerical experiments. In the following, the explicit dependence on the variables $(\bm{x}^\star, \bm{z}^\star(\xi),\bm{\delta}^\star)$ appearing in Eqs.~\eqref{eq:originalSystemsLin},~\eqref{eq:matricesTilde}, and~\eqref{eq:matricesH} will be suppressed to alleviate the notation.

\subsubsection{Spectral analysis}\label{app:Spectral}
The stability of the mild solution $(\tilde{\bm{x}},\tilde{\bm{z}}) \in C^0([0,T];\mathcal{X})$ of the linearised ODE-PDE system~\eqref{eq:originalSystemsLin} may be investigated starting with an analytical expression for the resolvent $\mathscr{R}(\lambda, \tilde{\mathscr{A}}) \triangleq (\lambda I_{\mathcal{X}}-\tilde{\mathscr{A}})^{-1}$ of the unbounded operator $(\tilde{\mathscr{A}},\mathscr{D}(\tilde{\mathscr{A}}))$ defined in Appendix~\ref{app:ProofLin}. Accordingly,
\begin{align}\label{eq:Atilde-}
\bigl((\tilde{\mathscr{A}}-\lambda I_{\mathcal{X}})(\bm{y}_1,\bm{\zeta}_1)\bigr)(\xi) = \begin{bmatrix}\bigl(\tilde{\mathbf{A}}_1-\lambda\mathbf{I}_2\bigr)\bm{y}_1 + \mathbf{G}_1(\tilde{\mathscr{K}}_1\bm{\zeta})  \\ -\mathbf{\Lambda}\dpd{\bm{\zeta}_1(\xi)}{\xi} + (\mathbf{\Sigma}-\lambda\mathbf{I}_2)\bm{\zeta}_1(\xi) + (\tilde{\mathscr{K}}_2\bm{\zeta})+\tilde{\mathbf{A}}_2(\xi)\bm{y}_1 \end{bmatrix} = \begin{bmatrix} \bm{y}_2 \\ \bm{\zeta}_2(\xi)\end{bmatrix},
\end{align}
where, for convenience of notation,
\begin{subequations}
\begin{align}
(\tilde{\mathscr{K}}_1\bm{\zeta}) & \triangleq (\mathscr{K}_1\bm{\zeta}) + \mathbf{\Sigma}(\mathscr{K}_2\bm{\zeta}), \\
(\tilde{\mathscr{K}}_2\bm{\zeta}) &\triangleq \mathbf{\Sigma}(\mathscr{K}_3\bm{\zeta}) + (\mathscr{K}_4\bm{\zeta}). 
\end{align}
\end{subequations}
Solving for the second component in Eq.~\eqref{eq:Atilde-} yields
\begin{align}\label{eq:SolZ111111}
\bm{\zeta}_1(\xi) & = \mathbf{\Gamma}(\xi,\lambda)\mathbf{\Lambda}^{-1}(\tilde{\mathscr{K}}_2\bm{\zeta}_1) + \mathbf{\Xi}(\xi,\lambda)\bm{y}_1+(\tilde{\mathscr{K}}_0\bm{\zeta}_2)(\xi,\lambda), \quad \xi \in [0,1],
\end{align}
where
\begin{subequations}
\begin{align}
\mathbf{\Gamma}(\xi,\lambda) & \triangleq \int_0^\xi\mathbf{\Phi}\bigl(\xi,\xi^\prime,\lambda\bigr) \dif \xi^\prime, \\
\mathbf{\Xi}(\xi,\lambda) & \triangleq \int_0^\xi\mathbf{\Phi}\bigl(\xi,\xi^\prime,\lambda\bigr)\mathbf{\Lambda}^{-1}\tilde{\mathbf{A}}_2(\xi^\prime) \dif \xi^\prime, \\
\mathbf{\Phi}(\xi,\tilde{\xi},\lambda) & \triangleq \eu^{\mathbf{\Lambda}^{-1}(\mathbf{\Sigma}-\lambda\mathbf{I}_2)(\xi-\tilde{\xi})},  
\end{align}
\end{subequations}
and
\begin{align}
(\tilde{\mathscr{K}}_0\bm{\zeta})(\xi,\lambda) & \triangleq-\int_0^\xi \mathbf{\Phi}\bigl(\xi,\xi^\prime,\lambda\bigr)\mathbf{\Lambda}^{-1}\bm{\zeta}(\xi^\prime) \dif \xi^\prime.
\end{align}
Consequently,
\begin{subequations}\label{eq:K1K2zetattt}
\begin{align}
(\tilde{\mathscr{K}}_1\bm{\zeta}_1) & = \mathbf{\Theta}_1(\lambda)(\tilde{\mathscr{K}}_2\bm{\zeta}_1) + \mathbf{\Psi}_1(\lambda)\bm{y}_1 + \bigl(\tilde{\mathscr{K}}_1(\tilde{\mathscr{K}}_0\bm{\zeta}_2)\bigr)(\lambda), \\
(\tilde{\mathscr{K}}_2\bm{\zeta}_1) & = \mathbf{\Theta}_2(\lambda)(\tilde{\mathscr{K}}_2\bm{\zeta}_1) + \mathbf{\Psi}_2(\lambda)\bm{y}_1 + \bigl(\tilde{\mathscr{K}}_2(\tilde{\mathscr{K}}_0\bm{\zeta}_2)\bigr)(\lambda), 
\end{align}
\end{subequations}
with
\begin{subequations}
\begin{align}
\mathbf{\Theta}_1(\lambda) & \triangleq (\tilde{\mathscr{K}}_1\mathbf{\Gamma})(\lambda)\mathbf{\Lambda}^{-1}, \\
\mathbf{\Theta}_2(\lambda) & \triangleq (\tilde{\mathscr{K}}_2\mathbf{\Gamma})(\lambda)\mathbf{\Lambda}^{-1}, \\
\mathbf{\Psi}_1(\lambda) & \triangleq(\tilde{\mathscr{K}}_1\mathbf{\Xi})(\lambda), \\
\mathbf{\Psi}_2(\lambda) & \triangleq(\tilde{\mathscr{K}}_2\mathbf{\Xi})(\lambda).
\end{align}
\end{subequations}
Combining the first component of Eq.~\eqref{eq:Atilde-} with the above~\eqref{eq:K1K2zetattt} provides
\begin{align}\label{eq:InvMMMMMMMMM}
\begin{bmatrix}\bm{y}_1 \\ (\tilde{\mathscr{K}}_1\bm{\zeta}_1) \\(\tilde{\mathscr{K}}_2\bm{\zeta}_1) \end{bmatrix} = \tilde{\mathbf{A}}^{-1}(\lambda)\begin{bmatrix} \bm{y}_2 \\ \bigl(\tilde{\mathscr{K}}_1(\tilde{\mathscr{K}}_0\bm{\zeta}_2)\bigr)(\lambda) \\  \bigl(\tilde{\mathscr{K}}_2(\tilde{\mathscr{K}}_0\bm{\zeta}_2)\bigr)(\lambda)\end{bmatrix},
\end{align}
with
\begin{align}\label{eq:matrixM}
\tilde{\mathbf{A}}(\lambda) & \triangleq \begin{bmatrix} \tilde{\mathbf{A}}_1-\lambda\mathbf{I}_2 & \mathbf{G}_1 & \mathbf{0} \\ -\mathbf{\Psi}_1(\lambda) & \mathbf{I}_2 & -\mathbf{\Theta}_1(\lambda) \\ -\mathbf{\Psi}_2(\lambda) & \mathbf{0} & \mathbf{I}_2-\mathbf{\Theta}_2(\lambda)\end{bmatrix}.
\end{align}
Recalling Eq.~\eqref{eq:SolZ111111}, it may be finally concluded that
\begin{align}
\begin{split}
\begin{bmatrix} \bm{y}_1 \\ \bm{\zeta}_1(\xi)\end{bmatrix} & = \bigl((\tilde{\mathscr{A}}-\lambda I_{\mathcal{X}})^{-1}(\bm{y}_2,\bm{\zeta}_2)\bigr)(\xi) =-\bigl(\mathscr{R}(\lambda,\tilde{\mathscr{A}})(\bm{y}_2,\bm{\zeta}_2)\bigr)(\xi) \\
& = \begin{bmatrix} \mathbf{I}_2 & \mathbf{0} & \mathbf{0} \\
\mathbf{\Xi}(\xi,\lambda) & \mathbf{0} & \mathbf{\Gamma}(\xi,\lambda)\mathbf{\Lambda}^{-1} \end{bmatrix}\tilde{\mathbf{A}}^{-1}(\lambda)\begin{bmatrix} \bm{y}_2 \\ \bigl(\tilde{\mathscr{K}}_1(\tilde{\mathscr{K}}_0\bm{\zeta}_2)\bigr)(\lambda) \\  \bigl(\tilde{\mathscr{K}}_2(\tilde{\mathscr{K}}_0\bm{\zeta}_2)\bigr)(\lambda)\end{bmatrix} + \begin{bmatrix} \mathbf{0} \\ (\tilde{\mathscr{K}}_0\bm{\zeta}_2)(\xi,\lambda)\end{bmatrix}, \quad \xi \in [0,1].
\end{split}
\end{align}
In particular, it may be realised that the condition $D(\lambda) \triangleq \det \tilde{\mathbf{A}}(\lambda) \not = 0$ for all $\lambda \in \mathbb{C}_{\geq 0}$, with $\tilde{\mathbf{A}}(\lambda) \in \mathbf{M}_6(\mathbb{C})$ as in~\eqref{eq:matrixM}, ensures that $\mathscr{R}(\lambda,\tilde{\mathscr{A}}) \in H^\infty(\mathbb{C}_{>0};\mathscr{L}(\mathcal{X}))$, which yields uniform exponential stability for the ODE-PDE system (Theorem 4.1.5 in \cite{Zwart}). Therefore, the stability of the system may be inferred by studying the roots of the characteristic equation $D(\lambda) = 0$, rather than analysing the entire resolvent. 

The condition on $D(\lambda)$ proves extremely useful when studying numerically the occurrence of micro-shimmy oscillations, which typically manifest at longitudinal speeds below 0.5 $\text{m}\,\text{s}^{-1}$. 
For instance, Fig.~\ref{fig:Charts} illustrates some stability charts constructed for two single-track models with constant pressure distribution and flexible tyre carcass, linearised around the zero equilibrium $(\bm{x}^\star, \bm{z}^\star(\xi), \bm{\delta}^\star) = \bm{0}$. Specifically, Fig.~\ref{fig:Charts} was produced considering continuously varying values of the ratio $\chi \triangleq C_{1}l_1/(C_{2}l_2)$ in the range between $0.5$ and $1.5$, where the cornering stiffnesses of the axle, $C_i$, $i \in \{1,2\}$, may be calculated as
\begin{align}
C_i \triangleq L_iF_{zi}\sigma_{0,i}, \quad i \in \{1,2\}.
\end{align}
The stability charts reported in Fig.~\ref{fig:Charts} are similar to those already obtained in \cite{BicyclePDE}, and display the first three unstable islands\footnote{In \cite{Takacs5}, neglecting the relaxation effect of the tyre carcass, the existence of a monotonically decreasing sequence of velocities for which oscillatory instabilities occur was deduced analytically, concerning neutral steer vehicles and employing a linear distributed bursh model to describe the tyre dynamics.} (in white) determined for two different values of the front and rear relaxation lengths $\lambda_{i}$, $i\in \{1,2\}$, computed according to
\begin{align}
\lambda_i \triangleq \dfrac{L_i(F_{zi}\sigma_{0,i} + w_i)}{2w_i}, \quad i \in \{1,2\}.
\end{align}
The unstable regions depicted in Fig.~\ref{fig:Charts} correspond to combination of parameters associated with oscillatory micro-shimmy dynamics.
As also noted in \cite{BicyclePDE}, lower values of the relaxation lengths seem to excite oscillatory behaviours on a wider range of longitudinal speeds, as it may be concluded by comparing Figs.~\ref{fig:Charts}(a) and~\ref{fig:Charts}(b). Moreover, whilst unstable dynamics seem to manifest in oversteer ($\chi > 1$), neutral steer ($\chi = 1$), and highly understeer ($\chi < 1$) vehicles, mild understeer conditions may possibly ensure stability for every value of the longitudinal speed $v_x$. This observation corroborates the findings of \cite{BicyclePDE}.

\begin{figure}
\centering
\subfloat[$\lambda_{1} = 0.195$ m, $\lambda_{2} = 0.225$ m.]{%
\resizebox*{16cm}{!}{\includegraphics{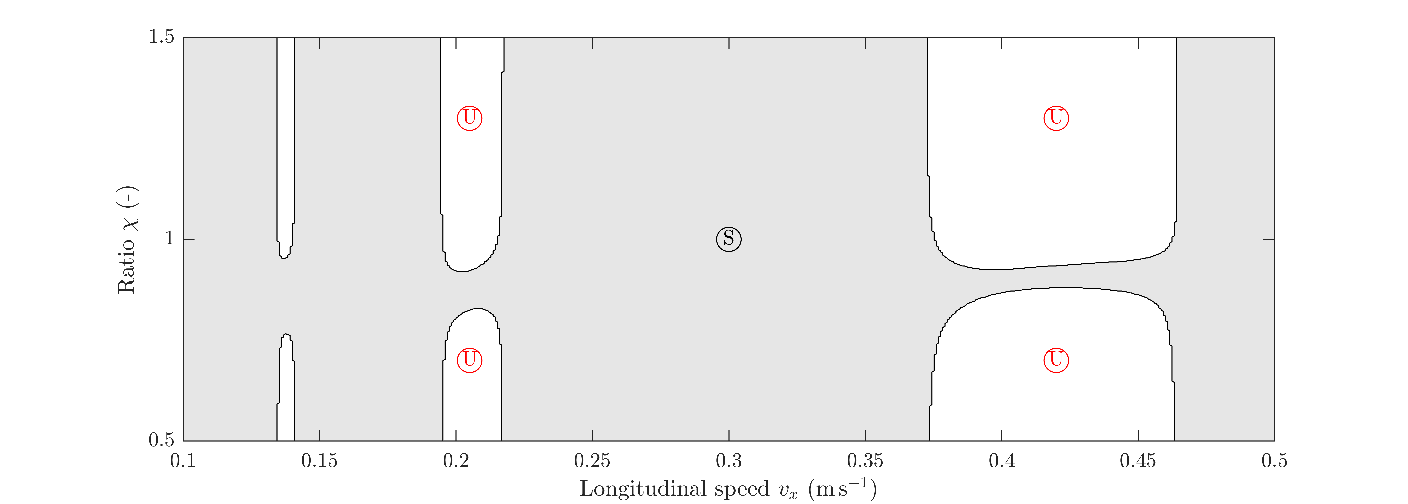}}}\hspace{5pt}
\subfloat[$\lambda_{1}= 0.390$ m, $\lambda_{2} = 0.450$ m.]{%
\resizebox*{16cm}{!}{\includegraphics{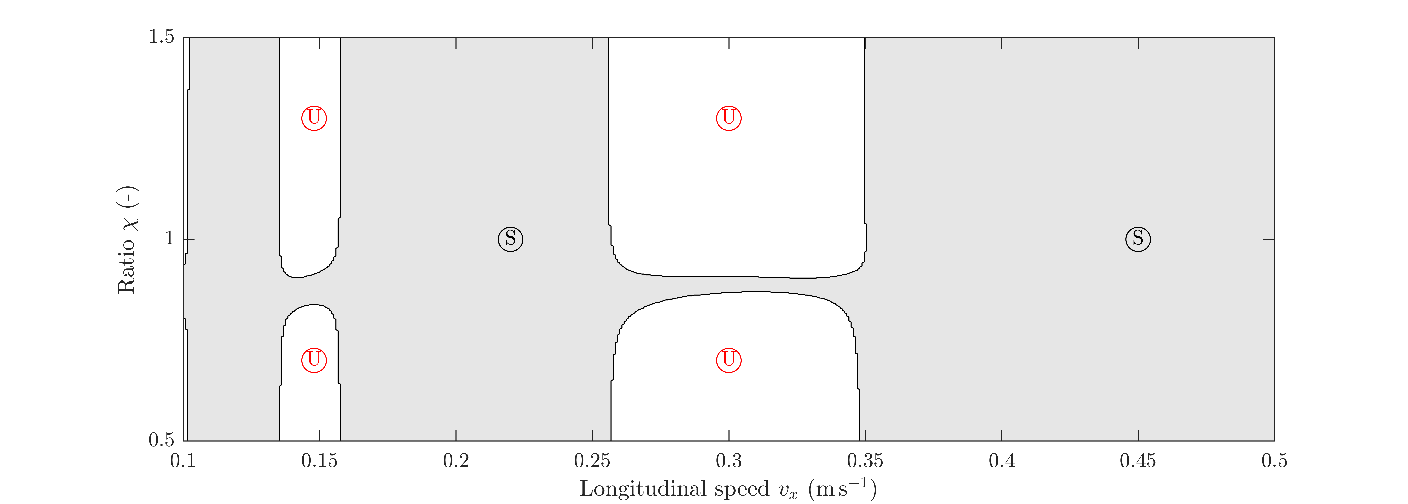}}}\hspace{5pt}
\caption{Stability charts for two single-track models with constant pressure distribution and flexible tyre carcass linearised around the zero equilibrium $(\bm{x}^\star, \bm{z}^\star(\xi), \bm{\delta}^\star) = \bm{0}$, for different values of the understeer index $\chi \triangleq C_{1}l_1/(C_{2}l_2)$ and longitudinal speed $v_x$. The unstable regions (in white) correspond to combinations of parameters for which the characteristic function $D(\lambda) = \det \tilde{\mathbf{A}}(\lambda)$ has two roots with positive real part. Model parameters: $m = 1300$ kg, $I_z = 2000$ $\text{kg}\,\text{m}^2$, $l_1  =1$ m, $l_2 = 1.6$ m, $L_1 = 0.11$ m, $L_2 = 0.9$ m, $\sigma_{0,1} = 163$ $\textnormal{m}^{-1}$, $\sigma_{0,2} = 408$ $\textnormal{m}^{-1}$, $F_{z1} = 3924$ N, $F_{z2} = 2453$ N; (a) $\lambda_{1} = 0.195$ m, $\lambda_{2} = 0.225$ m; (b) $\lambda_{1} = 0.390$ m, $\lambda_{2} = 0.450$ m.} \label{fig:Charts}
\end{figure}
At high longitudinal speeds $v_x$, the behaviour of the linearised single-track models essentially coincides with that predicted by the standard formulations equipped with static descriptions of the tyre dynamics. Indeed, for sufficiently large values of $v_x$, the rolling contact process evolves on time scales that are substantially shorter than those of the rigid chassis motion, and may be fairly neglected. As a consequence, non-oscillatory instabilities are only detected in oversteer vehicles ($\chi > 1$) beyond a critical value of the longitudinal speed, which solely depends on the vehicle's structural parameters. On this matter, further details may be found in \cite{BicyclePDE}.

\subsubsection{System's output and transfer function}\label{app:output}
The analysis conducted in the previous Sect.~\ref{app:Spectral} has offered a complete characterisation of the spectral properties of the operator $(\tilde{\mathscr{A}},\mathscr{D}(\tilde{\mathscr{A}}))$, providing an explicit criterion to deliberate about the stability of the linearised single-track models described by Eqs.~\eqref{eq:originalSystemsLin},~\eqref{eq:matricesTilde}, and~\eqref{eq:matricesH}. The results advocated in Sect.~\ref{app:Spectral} may also be exploited to investigate the frequency response of the ODE-PDE system ~\eqref{eq:originalSystemsLin} for a wide range of operating conditions. In this context, in most applications, in addition to the knowledge of the lumped states $\tilde{\bm{x}}(t)$, it is also of interest to study the behaviour of the perturbed axle forces and normalised lateral acceleration. In particular, the vector collecting the perturbations of the axle forces, $\mathbb{R}^2 \ni \tilde{\bm{F}}_y(t) = [\tilde{F}_{y1}(t) \; \tilde{F}_{y2}(t)]^{\mathrm{T}}$, may be calculated as
\begin{align}\label{eq:Ftilde}
\tilde{\bm{F}}_y(t) & = \mathbf{H}_1\mathbf{A}_2\tilde{\bm{x}}(t) + (\mathscr{K}_1\tilde{\bm{z}})(t) + \mathbf{\Sigma}(\mathscr{K}_2\tilde{\bm{z}})(t)+\tilde{\mathbf{B}}_1\tilde{\bm{\delta}}(t) = \mathbf{H}_1\mathbf{A}_2\tilde{\bm{x}}(t) + (\tilde{\mathscr{K}}_1\tilde{\bm{z}})(t) + \tilde{\mathbf{B}}_1\tilde{\bm{\delta}}(t),
\end{align}
and, consequently, the perturbed normalised acceleration may be inferred to be
\begin{align}
\dfrac{\tilde{a}_y(t)}{g} & = -\dfrac{1}{m g}\bigl(\tilde{F}_{y1}(t) + \tilde{F}_{y2}(t)\bigr).
\end{align}
Therefore, the output vector $\bm{y}(t) \in \mathbb{R}^5$ may be assembled as
\begin{align}\label{eq:outputY}
\bm{y}(t) & \triangleq \mathbf{C}\begin{bmatrix}\tilde{\bm{x}}(t) \\ \tilde{\bm{F}}_y(t) \end{bmatrix},
\end{align}
with
\begin{align}\label{eq:matrixC}
\mathbf{C} & \triangleq \begin{bmatrix} 1 & 0 & 0 & 0\\
0 & 1 & 0 & 0 \\
0 & 0 & 1 & 0 \\
0 & 0 & 0 & 1 \\
0 & 0 & -\dfrac{1}{mg} & -\dfrac{1}{mg}\end{bmatrix}.
\end{align}
The analytical expression for the transfer function $\mathbf{G}_{\widehat{\tilde{\bm{\delta}}}(s) \to \widehat{\bm{y}}(s)}$ from $\widehat{\tilde{\bm{\delta}}}(s) \triangleq (\mathcal{L}\tilde{\bm{\delta}})(s)$ to $\widehat{\bm{y}}(s) \triangleq (\mathcal{L}\bm{y})(s)$ is given in Lemma~\ref{lemma:Transfer}.

\begin{lemma}[Transfer function]\label{lemma:Transfer}
The transfer function $\mathbf{G}_{\widehat{\tilde{\bm{\delta}}}(s) \to \widehat{\bm{y}}(s)}$ from $\widehat{\tilde{\bm{\delta}}}(s) \triangleq (\mathcal{L}\tilde{\bm{\delta}})(s)$ to $\widehat{\bm{y}}(s) \triangleq (\mathcal{L}\bm{y})(s)$ is given by
\begin{align}\label{eq:transferFUn}
\mathbf{G}_{\widehat{\tilde{\bm{\delta}}}(s) \to \widehat{\bm{y}}(s)}(s) = -\mathbf{C}\biggggl(\begin{bmatrix} \mathbf{I}_2 & \mathbf{0} & \mathbf{0} \\
\mathbf{H}_1\mathbf{A}_2 & \mathbf{I}_2 & \mathbf{0}\end{bmatrix}\tilde{\mathbf{A}}^{-1}(s)\begin{bmatrix} \mathbf{G}_1\tilde{\mathbf{B}}_1 \\  \bigl(\tilde{\mathscr{K}}_1(\tilde{\mathscr{K}}_0\tilde{\mathbf{B}}_2)\bigr)(s) \\  \bigl(\tilde{\mathscr{K}}_2(\tilde{\mathscr{K}}_0\tilde{\mathbf{B}}_2)\bigr)(s)\end{bmatrix} -\begin{bmatrix}\mathbf{0} \\ \tilde{\mathbf{B}}_1 \end{bmatrix}\biggggr).
\end{align}
If $D(s) \triangleq \det \tilde{\mathbf{A}}(s) \not = 0$ for all $s \in \mathbb{C}_{\geq 0}$, then the transfer function in Eq.~\eqref{eq:transferFUn} is stable.
\begin{proof}
Combining Eqs.~\eqref{eq:Ftilde} and~\eqref{eq:outputY} yields
\begin{align}\label{eq:YIntermed}
\widehat{\bm{y}}(s) = \mathbf{C}\bigggl(\begin{bmatrix}\mathbf{I}_2 & \mathbf{0}\\ \mathbf{H}_1\mathbf{A}_2 & \mathbf{I}_2\end{bmatrix}\begin{bmatrix} \widehat{\tilde{\bm{x}}}(s) \\ (\tilde{\mathscr{K}}_1\widehat{\tilde{\bm{z}}})(s)\end{bmatrix} + \begin{bmatrix} \mathbf{0} \\ \tilde{\mathbf{B}}_1\end{bmatrix}\widehat{\tilde{\bm{\delta}}}(s)\bigggr), 
\end{align}
where $(\widehat{\tilde{\bm{x}}}, \widehat{\tilde{\bm{z}}})(\cdot,s) \triangleq (\mathcal{L}(\tilde{\bm{x}},\tilde{\bm{z}}))(\cdot,s)$, and $\widehat{\tilde{\bm{\delta}}}(s) \triangleq (\mathcal{L}\tilde{\bm{\delta}})(s)$. From Eq. ~\eqref{eq:InvMMMMMMMMM} with $\bm{y}_1 \triangleq \widehat{\tilde{\bm{x}}}(s)$, $\bm{\zeta}_1(\xi) \triangleq \widehat{\tilde{\bm{z}}}(\xi,s)$, $\bm{y}_2 = -\mathbf{G}_1\tilde{\mathbf{B}}_1\widehat{\tilde{\bm{\delta}}}(s)$, and $\bm{\zeta}_1(\xi) = -\tilde{\mathbf{B}}_2(\xi)\widehat{\tilde{\bm{\delta}}}(s)$ it also follows that
\begin{align}\label{eq:InvMMMMMMMMM2}
\begin{bmatrix} \widehat{\tilde{\bm{x}}}(s) \\ (\tilde{\mathscr{K}}_1\widehat{\tilde{\bm{z}}})(s)\end{bmatrix} =-\begin{bmatrix}\mathbf{I}_2 & \mathbf{0} & \mathbf{0} \\ \mathbf{0} & \mathbf{I}_2 & \mathbf{0} \end{bmatrix} \tilde{\mathbf{A}}^{-1}(s)\begin{bmatrix} \mathbf{G}_1\tilde{\mathbf{B}}_1 \\ \bigl(\tilde{\mathscr{K}}_1(\tilde{\mathscr{K}}_0\tilde{\mathbf{B}}_2)\bigr)(s) \\  \bigl(\tilde{\mathscr{K}}_2(\tilde{\mathscr{K}}_0\tilde{\mathbf{B}}_2)\bigr)(s)\end{bmatrix}\widehat{\tilde{\bm{\delta}}}(s).
\end{align}
Inserting Eq.~\eqref{eq:InvMMMMMMMMM2} into~\eqref{eq:YIntermed} produces~\eqref{eq:transferFUn}. The second claim is a consequence of the findings of Sect.~\ref{app:Spectral}.
\end{proof}
\end{lemma}

Equation~\eqref{eq:transferFUn} allows investigating the frequency response of the linearised single-track models over a wide range of operating conditions. In this context, the typical frequency response of a single-track model with flexible tyre carcass, linearised around the zero equilibrium $(\bm{x}^\star, \bm{z}^\star(\xi), \bm{\delta}^\star) = \bm{0}$, is illustrated in Fig.~\ref{fig:BodeLuGre}, where the Bode diagrams (in semilog scale) from $\widehat{\tilde{\delta}}_1(s) \triangleq (\mathcal{L}\tilde{\delta}_1)(s)$ to the components of $\widehat{\bm{y}}(s)$ as defined in Eqs.~\eqref{eq:outputY} and~\eqref{eq:matrixC} are collected for three different values of the longitudinal speed $v_x = 20$, 40, and 60 $\text{m}\,\text{s}^{-1}$ (blue, orange, and yellow lines, respectively). In particular, it may be observed that both the amplitudes and the phases of the output quantities depicted in Fig.~\ref{fig:BodeLuGre} decrease nonlinearly and at increasing rates at large frequencies. As extensively discussed in \cite{BicyclePDE}, these trends differ significantly from those obtained using standard linear single-track models with static tyres, which predict the existence of asymptotic values approached at around $\omega \approx 100$ Hz. 
\begin{figure}
\centering
\includegraphics[width=1\linewidth]{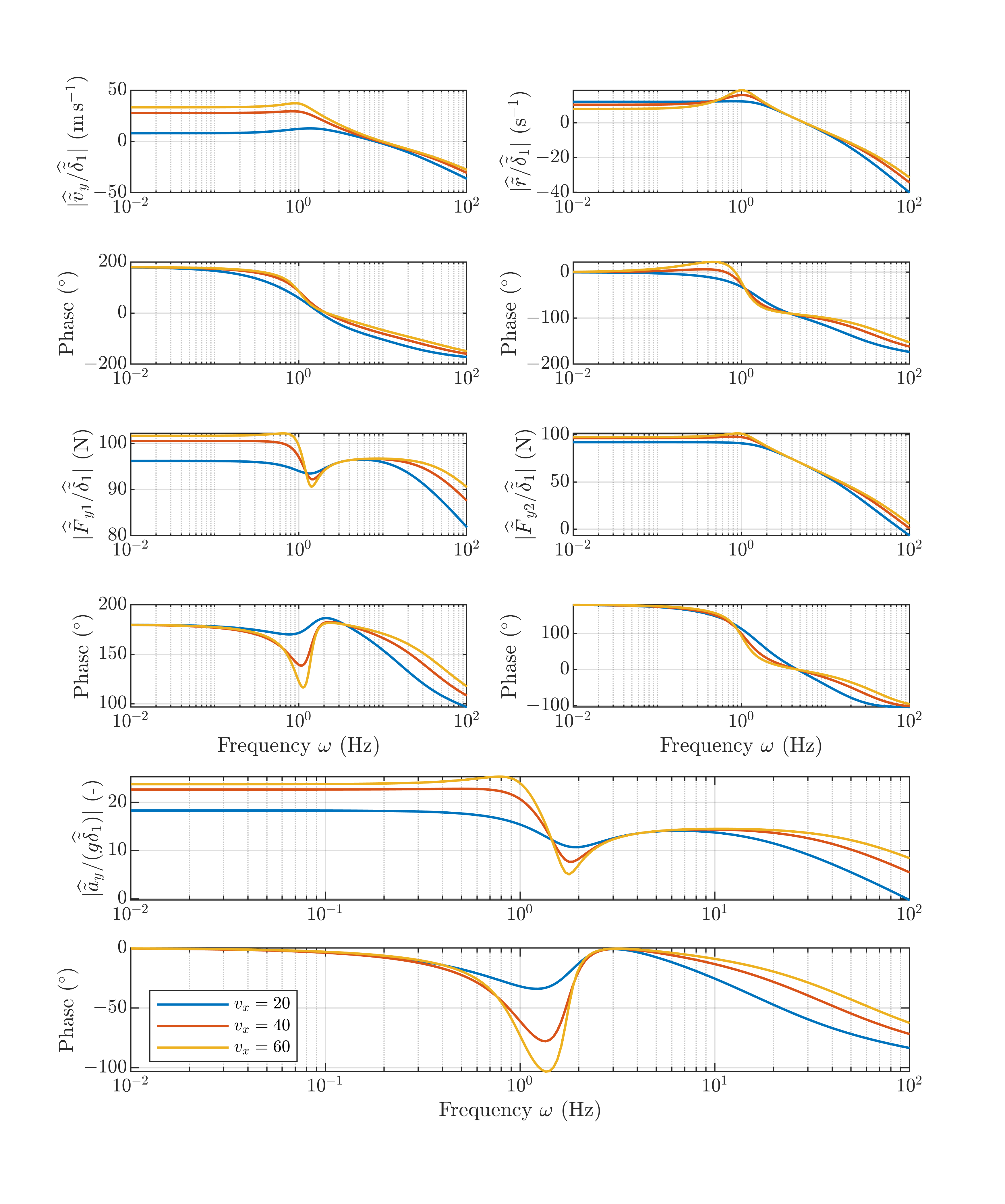} 
\caption{Typical Bode diagrams (in semilog scale) from $\widehat{\tilde{\delta}}_1(s)$ to the components of $\widehat{\bm{y}}(s)$ for a single-track model linearised around the zero equilibrium $(\bm{x}^\star, \bm{z}^\star(\xi), \bm{\delta}^\star) = \bm{0}$, for different values of the longitudinal speed $v_x = 20$, $40$, and $60$ $\text{m}\,\text{s}^{-1}$. Model parameters as in Fig.~\ref{fig:Charts}(a).}
\label{fig:BodeLuGre}
\end{figure}

\section{Numerical simulations}\label{sect:Simulations}
The present section investigates the dynamic behaviour of the semilinear single-track models developed in the paper, considering both micro-shimmy oscillations (Sect.~\ref{sect:shimmy}) and typical steering manoeuvres (Sect.~\ref{sect:steering}). The results reported below refer to numerical simulations conducted in MATLAB/Simulink\textsuperscript{\textregistered} environment, obtained using a finite-difference scheme for the space discretisation of the PDEs, and with discretisation steps of $0.02$ and $10^{-4}$ s for the spatial and temporal resolutions, respectively. All the simulations were conducted with $\abs{v}_\varepsilon \triangleq \sqrt{v^2 + \varepsilon}$ and $\varepsilon = 10^{-6}$ $\textnormal{m}^2\,\textnormal{s}^{-2}$, and subsequently repeated with $\varepsilon = 0$, with no appreciable difference.

\subsection{Micro-shimmy oscillations}\label{sect:shimmy}
A prerogative of single-track models with distributed tyre dynamics consists in their ability to predict the existence of micro-shimmy oscillations, typically occurring at low longitudinal speeds $v_x \leq 5$ $\textnormal{m}\,\textnormal{s}^{-1}$. Therefore, it is first interesting to investigate the effect of different model parameters on the characteristics of such self-excited vibrations. Figure~\ref{fig:Shimmy1} was produced considering two different semilinear single-track models driving at $v_x = 0.45$ $\textnormal{m}\,\textnormal{s}^{-1}$, with rigid tyre carcass and constant pressure distribution, but different values for the micro-damping coefficients $\sigma_{1,i}$, $i \in \{1,2\}$. In particular, concerning the tyre forces, both formulations exhibit an initial oscillatory dynamics with relatively large amplitudes, which decay over time due to the dissipative effects induced by the nonlinear friction model. 
\begin{figure}
\centering
\includegraphics[width=0.9\linewidth]{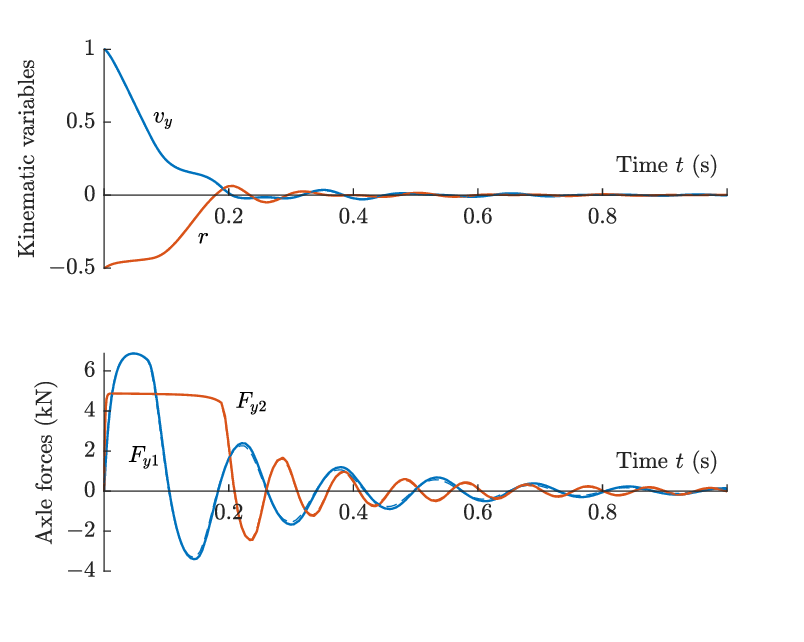} 
\caption{Spontaneous micro-shimmy vibrations for two semilinear single-track models with rigid carcass, constant pressure distribution, and different damping coefficients, driving at $v_x = 0.45$ $\textnormal{m}\,\textnormal{s}^{-1}$. Solid lines: $\sigma_{1,1} = \sigma_{1,2} =0$ $\textnormal{s}\,\textnormal{m}^{-1}$; Dashed lines: $\sigma_{1,1} = \sigma_{1,2} =0.1$ $\textnormal{s}\,\textnormal{m}^{-1}$. Other model parameters as in Table~\ref{tab:parametersSimul}.}
\label{fig:Shimmy1}
\end{figure}

The micro-shimmy behaviours of two variants with rigid (solid lines) and flexible tyre carcasses (dashed lines) are compared in Fig.~\ref{fig:Shimmy2}. In the second case, the additional nonlocal phenomena excited by the structural compliance of the tyre appear to exacerbate the oscillatory dynamics of the vehicle, producing larger amplitudes and longer transients concerning both the kinematic variables and axle forces. 
\begin{figure}
\centering
\includegraphics[width=0.9\linewidth]{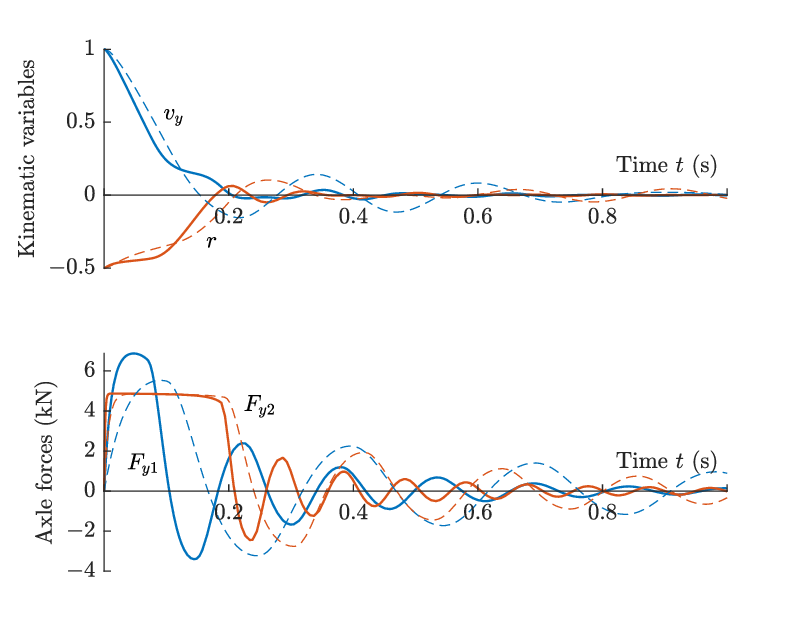} 
\caption{Spontaneous micro-shimmy vibrations for two semilinear single-track models with rigid and flexible carcass, and constant pressure distribution, driving at $v_x = 0.45$ $\textnormal{m}\,\textnormal{s}^{-1}$. Solid lines: model with rigid carcass and $\sigma_{1,1} = \sigma_{1,2} =0$ $\textnormal{s}\,\textnormal{m}^{-1}$; Dashed lines: model with flexible carcass with $w_1 = w_2 = 2.5\cdot 10^6$ N, and $\sigma_{1,1} = \sigma_{1,2} =0$ $\textnormal{s}\,\textnormal{m}^{-1}$. Other model parameters as in Table~\ref{tab:parametersSimul}.}
\label{fig:Shimmy2}
\end{figure}

Finally, the role of different pressure distributions on the entity of the micro-shimmy oscillations is investigated in Fig.~\ref{fig:Shimmy3}, considering two semilinear single-track models with flexible carcass. Specifically, it may be observed that the formulation with constant profile (solid lines) predicts slightly larger amplitudes compared to that with exponentially decreasing contact pressure (dashed lines). This is consistent with the observation that an exponentially decreasing distribution often implies strict-dissipativity properties for the distributed FrBD model \cite{FrBD,DistrLuGre}, eventually robustifying the stability of the vehicle. However, the behaviors illustrated in Fig.~\ref{fig:Shimmy3} are essentially analogous, suggesting that the shape of the contact pressure might play only a minor role in determining the stability of the semilinear single-track models.
\begin{figure}
\centering
\includegraphics[width=0.9\linewidth]{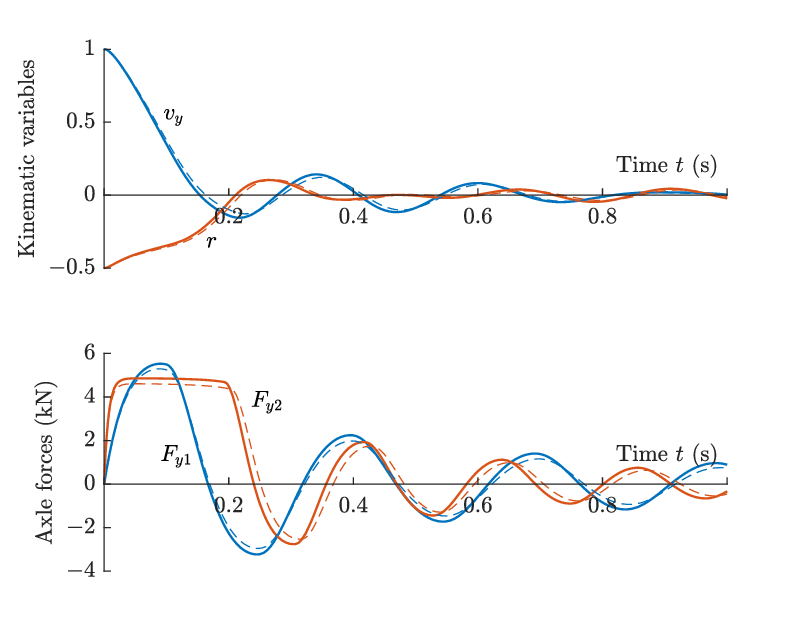} 
\caption{Spontaneous micro-shimmy vibrations for two semilinear single-track models with flexible carcass and different pressure distribution, driving at $v_x = 0.45$ $\textnormal{m}\,\textnormal{s}^{-1}$. Solid lines: model with constant pressure distribution; Dashed lines: model with exponentially decreasing pressure distribution, with $a_1 = a_2 = 1$. Other model parameters as in Table~\ref{tab:parametersSimul}.}
\label{fig:Shimmy3}
\end{figure}
The model parameters used to produce Figs.~\ref{fig:Shimmy1},~\ref{fig:Shimmy2}, and~\ref{fig:Shimmy3} are listed in Table~\ref{tab:parametersSimul}.

\begin{table}[h!]\centering 
\caption{Model parameters}
{\begin{tabular}{|c|c|c|c|}
\hline
Parameter & Description & Unit & Value \\
\hline 
$m$ & Vehicle mass & kg & 1300 \\ 
$I_z$ & Vertical moment of inertia  & $\textnormal{kg}\,\textnormal{m}^{2}$ & 2000 \\
$l_1$ & Front axle length & m & 1  \\
$l_2$ & Rear axle length & m & 1.6 \\
$L_1$ & Front contact patch length & m & 0.11 \\
$L_2$ & Rear contact patch length & m & 0.09 \\
$\sigma_{0,1}$ &Front micro-stiffness & $\textnormal{m}^{-1}$ & 163 \\
$\sigma_{0,2}$ & Rear micro-stiffness & $\textnormal{m}^{-1}$ & 408 \\
$\sigma_{1,1}$ &Front micro-damping & $\textnormal{s}\,\textnormal{m}^{-1}$ & 0.1 \\
$\sigma_{1,2}$ & Rear micro-damping & $\textnormal{s}\,\textnormal{m}^{-1}$ & 0.1 \\
$\sigma_{2,1}$ &Front viscous damping & $\textnormal{s}\,\textnormal{m}^{-1}$ & 0 \\
$\sigma_{2,2}$ & Rear viscous damping & $\textnormal{s}\,\textnormal{m}^{-1}$ & 0 \\
$\mu_1(\cdot)$ & Front friction coefficient & -& 1\\
$\mu_2(\cdot)$ & Rear friction coefficient & -& 1\\
$w_1$ & Front tyre carcass stiffness &N & $2.5\cdot 10^6$ \\
$w_2$ & Rear tyre carcass stiffness & N &$2.5\cdot 10^6$ \\
$F_{z1}$ & Front vertical tyre force & N & 3924 \\
$F_{z2}$ & Rear vertical tyre force  & N  & 2453 \\
$\chi_1$ & First parametrisation coefficient & -& 0\\
$\chi_2$ & Second parametrisation coefficient & -& 0\\
$\chi_3$ & Third parametrisation coefficient & -& 0\\
$\varepsilon$ & Regularisation parameter & $\textnormal{m}^2\,\textnormal{s}^{-2}$ & $10^{-6}$ \\ 
\hline
\end{tabular} }
\label{tab:parametersSimul}
\end{table}

\subsection{Steering manoeuvres}\label{sect:steering}

Next, the dynamical behaviour predicted by the semilinear single-track models is investigated considering typical cornering manoeuvres at normal cruising speeds. More specifically, the cases of constant steering and sine-sweep inputs are analysed qualitatively in the following.

Figure~\ref{fig:Steer1} compares the transient dynamics of the two single-track formulations with rigid (solid lines) and flexible carcass (dashed lines), and constant pressure distribution, obtained by prescribing constant steering angles $\delta_1(t) = \delta_1 = 2^\circ$ and $\delta_2(t) = \delta_2 = 0$. The trend illustrated in Fig.~\ref{fig:Steer1} refers to zero initial conditions, and a constant longitudinal speed $v_x = 20$ $\text{m}\,\text{s}^{-1}$. It may be easily observed that both variants predict convergence of the kinematic variables and axle forces to their steady-state values, which are reached approximately around $t = 0.6$ s. Generally speaking, the model with a compliant carcass exhibits slightly smoother dynamics, especially concerning the evolution of the axle forces.
\begin{figure}
\centering
\includegraphics[width=0.9\linewidth]{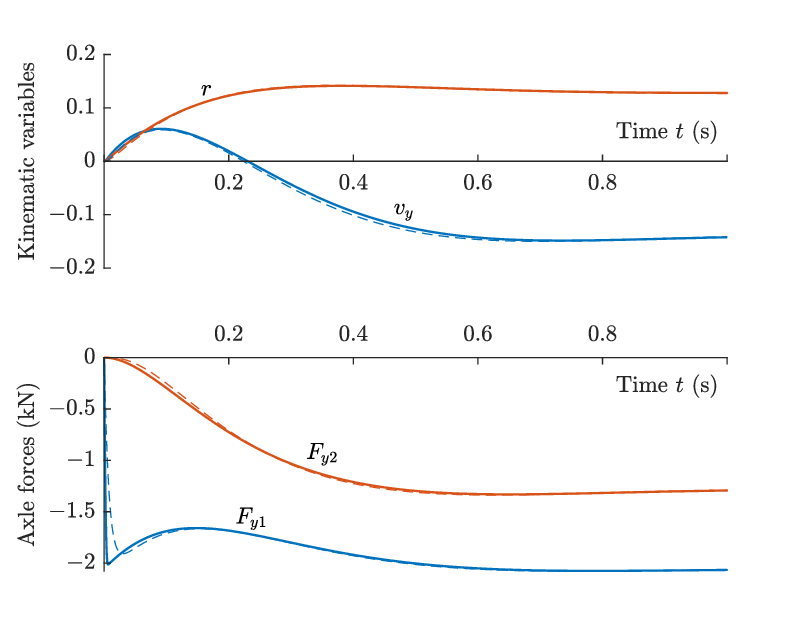} 
\caption{Dynamic response predicted by the semilinear single-track model with rigid (solid lines) and flexible carcass (dashed lines), for a vehicle travelling at a constant longitudinal speed $v_x = 20$ $\text{m}\,\text{s}^{-1}$, subjected to a cornering manoeuver with constant steering inputs $\delta_1 = 2^\circ$ and $\delta_2 = 0$. Solid lines: model with rigid carcass and $\sigma_{1,1} = \sigma_{1,2} =0$ $\textnormal{s}\,\textnormal{m}^{-1}$; Dashed lines: model with flexible carcass with $w_1 = w_2 = 2.5\cdot 10^6$ N, and $\sigma_{1,1} = \sigma_{1,2} =0$ $\textnormal{s}\,\textnormal{m}^{-1}$. Other model parameters as in Table~\ref{tab:parametersSimul}.}
\label{fig:Steer1}
\end{figure}

Similar considerations may be drawn by inspecting Fig.~\ref{fig:Steer2}, produced considering a sine-sweep input $\delta_1(t) = \bar{\delta}_1\sin(\omega t)$, with $\bar{\delta}_1 = 2^\circ$ and $\omega = 2$ $\textnormal{rad}\,\textnormal{s}^{-1}$, and again $\delta_2(t) = \delta_2 = 0$. In this case, the low frequency of excitation, combined with the small steering input, yields a better qualitative agreement between the variants with rigid and flexible carcasses. 
\begin{figure}
\centering
\includegraphics[width=0.9\linewidth]{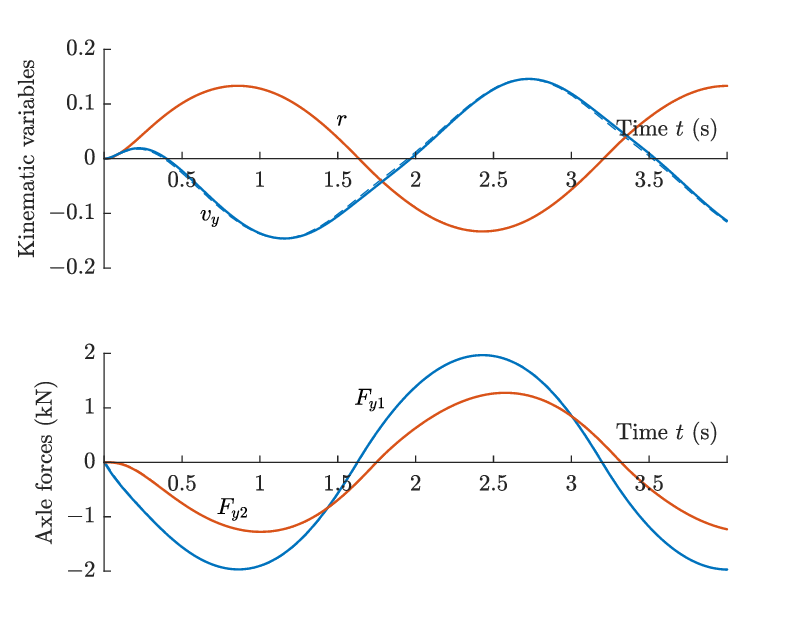} 
\caption{Dynamic response predicted by the semilinear single-track model with rigid (solid lines) and flexible carcass (dashed lines), for a vehicle travelling at a constant longitudinal speed $v_x = 20$ $\text{m}\,\text{s}^{-1}$, subjected to a sine-sweep manoeuver with amplitudes $\delta_1 = 2^\circ$, $\delta_2 = 0$, and steering frequency $\omega = 2$ $\textnormal{rad}\,\textnormal{s}^{-1}$. Solid lines: model with rigid carcass and $\sigma_{1,1} = \sigma_{1,2} =0$ $\textnormal{s}\,\textnormal{m}^{-1}$; Dashed lines: model with flexible carcass with $w_1 = w_2 = 2.5\cdot 10^6$ N, and $\sigma_{1,1} = \sigma_{1,2} =0$ $\textnormal{s}\,\textnormal{m}^{-1}$. Other model parameters as in Table~\ref{tab:parametersSimul}.}
\label{fig:Steer2}
\end{figure}

Different is instead the situation illustrated in Fig.~\ref{fig:Steer3}, which refers to a sine-sweep manoeuver designed with $\bar{\delta}_1 = 1^\circ$, $\delta_2(t) = 0$, and $\omega = 50$ $\textnormal{rad}\,\textnormal{s}^{-1}$. Combined with the relatively large steering input at the front wheels, the fast excitation induces distinct behaviours in the two semilinear single-track models, with the variant with compliant carcass exhibiting a delayed response compared to the simpler formulation. 
\begin{figure}
\centering
\includegraphics[width=0.9\linewidth]{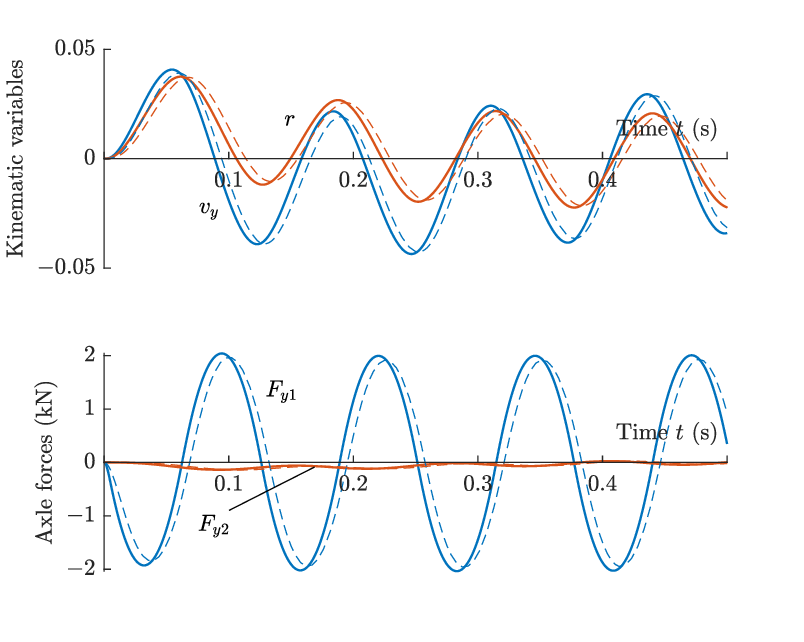} 
\caption{Dynamic response predicted by the semilinear single-track model with rigid (solid lines) and flexible carcass (dashed lines), for a vehicle travelling at a constant longitudinal speed $v_x = 20$ $\text{m}\,\text{s}^{-1}$, subjected to a sine-sweep manoeuver with amplitudes $\delta_1 = 1^\circ$, $\delta_2 = 0$, and steering frequency $\omega =50$ $\textnormal{rad}\,\textnormal{s}^{-1}$. Solid lines: model with rigid carcass and $\sigma_{1,1} = \sigma_{1,2} =0$ $\textnormal{s}\,\textnormal{m}^{-1}$; Dashed lines: model with flexible carcass with $w_1 = w_2 = 2.5\cdot 10^6$ N, and $\sigma_{1,1} = \sigma_{1,2} =0$ $\textnormal{s}\,\textnormal{m}^{-1}$. Other model parameters as in Table~\ref{tab:parametersSimul}.}
\label{fig:Steer3}
\end{figure}

\section{Conclusions}\label{sect:Conclusion}
This paper introduced a novel family of semilinear single-track vehicle models equipped with distributed tyre representations that incorporate the effect of finite friction. Specifically, the friction model adopted in this work is the \emph{Friction with Bristle Dynamics} (FrBD) one, which generalises famous formulations such as Dahl and LuGre by describing the distributed rolling contact process using a nonlinear \emph{partial differential equation} (PDE). In contrast to classic discontinuous Coulomb-like models, the FrBD description does not explicitly differentiate between local stick-slip behaviours inside the contact area, thereby overcoming the well-known drawbacks associated with non-smooth dynamics that render the tyre-road interaction problem difficult to analyse. 
This model was easily integrated into a single-track vehicle framework, yielding a semilinear ODE-PDE system that captures the interaction between the rigid-body lateral vehicle motion and distributed tyre dynamics. Two main variants of the single-track model were developed, each admitting a compact state-space representation. The first version assumes a rigid tyre carcass, whereas the second one considers the effect of a flexible carcass but disregards secondary contributions relating to damping and viscous terms. Local and global well-posedness properties for the coupled system were established rigorously, highlighting the dissipative and physically consistent properties of the distributed FrBD model. A linearisation procedure was also presented, enabling spectral analysis and transfer function derivation, and potentially facilitating the synthesis of controllers and observers. The performance of the newly developed single-track models was exemplified in simulation, considering both micro-shimmy scenarios and standard steering manoeuvres. In particular, the adoption of a semilinear variant qualitatively demonstrated the stabilising effect of friction on micro-shimmy oscillations arising at low longitudinal speeds. The proposed formulation advances the state-of-the-art modelling in vehicle dynamics by providing a physically grounded, mathematically rigorous, and computationally tractable approach to incorporating transient tyre behaviour in lateral vehicle dynamics, when accounting for the effect of limited friction. Future work could focus on investigating the stability of the semilinear single-track models around nontrivial equilibria, as well as on synthesising control and estimation algorithms based directly on the distributed parameter representations introduced in this paper. Moreover, full nonlinear descriptions, including additional contributions from suspension dynamics and load transfers, could be further developed, providing a more realistic and accurate representation of the intricate interaction between rigid chassis motion and distributed tyre phenomena. 

\section*{Funding declaration}
This research was financially supported by the project FASTEST (Reg. no. 2023-06511), funded by the Swedish Research Council. 


\section*{Compliance with Ethical Standards}

The authors declare that they have no conflict of interest.

\section*{Author Contribution declaration}
L.R. developed the models, carried out their mathematical analysis, performed all simulations, generated the figures, and wrote the manuscript. O.M.A., J.Å., and E.F. provided valuable feedback. All authors reviewed and approved the final version of the manuscript.

\appendix

\section{Technical results and proofs}\label{app:theorems}
The present Appendix collects some technical results and proofs.
\subsection{Proof of Theorems~\ref{thm:mild} and~\ref{thm:class}}\label{app:WellP}
Some technical results need first to be established before giving the proof of Theorems~\ref{thm:mild} and~\ref{thm:class}.

The ODE-PDE system~\eqref{eq:originalSystems} may be recast in abstract form as
\begin{subequations}\label{eq:abstracto}
\begin{align}
\dod{}{t}\begin{bmatrix} \bm{x}(t) \\ \bm{z}(t) \end{bmatrix} & = \mathscr{A}\begin{bmatrix} \bm{x}(t) \\ \bm{z}(t) \end{bmatrix} + \bm{f}\bigl(\bm{x}(t),\bm{z}(t),t\bigr), \quad t \in (0,T), \\
\begin{bmatrix} \bm{x}(0) \\ \bm{z}(0) \end{bmatrix} & = \begin{bmatrix} \bm{x}_0 \\ \bm{z}_0 \end{bmatrix}, 
\end{align}\label{eq:abstract}
\end{subequations}
where $(\mathscr{A},\mathscr{D}(\mathscr{A}))$ is the unbounded operator defined by
\begin{subequations}
\begin{align}\label{eq:Adecomp}
\bigl(\mathscr{A}(\bm{y},\bm{\zeta})\bigr)(\xi) & \triangleq \bigl(\mathscr{A}_0(\bm{y},\bm{\zeta})\bigr)(\xi) + \bigl(\mathscr{A}_1(\bm{y},\bm{\zeta})\bigr)(\xi), \\
\mathscr{D}(\mathscr{A}) & \triangleq \left\{ (\bm{y},\bm{\zeta}) \in \mathcal{Y} \mathrel{|} \bm{\zeta}(0) = \bm{0}\right\}, 
\end{align}
\end{subequations}
with $(\mathscr{A}_0, \mathscr{D}(\mathscr{A}_0))$ reading
\begin{subequations}\label{operator:A_0}
\begin{align}
\bigl(\mathscr{A}_0(\bm{y},\bm{\zeta})\bigr)(\xi) & \triangleq \begin{bmatrix} \mathbf{A}_1 \bm{y} + \mathbf{G}_1\mathbf{K}_2\bm{\zeta}(1) \\ -\mathbf{\Lambda} \dpd{\bm{\zeta}(\xi)}{\xi} + \mathbf{K}_6\bm{\zeta}(1)\end{bmatrix}, \\
\mathscr{D}(\mathscr{A}_0) & = \mathscr{D}(\mathscr{A}),
\end{align}
\end{subequations}
and $\mathscr{A}_1\in \mathscr{B}(\mathcal{X}) $ given by
\begin{align}
\bigl(\mathscr{A}_1(\bm{y},\bm{\zeta})\bigr)(\xi) \triangleq \begin{bmatrix} \mathbf{G}_1\int_0^1 \mathbf{K}_1(\xi)\bm{\zeta}(\xi)\dif \xi \\ \int_0^1 \mathbf{K}_5(\xi)\bm{\zeta}(\xi)\dif \xi \end{bmatrix}.
\end{align}
Finally, the nonlinear term appearing in~\eqref{eq:abstract}, $\bm{f} :\mathcal{X}\times [0,T]\to \mathcal{X}$, reads
\begin{align}
\bm{f}(\bm{x},\bm{z},t) \triangleq \begin{bmatrix} \mathbf{G}_1\mathbf{\Sigma}\Bigl(\bm{v}\bigl(\bm{x},\bm{\delta}(t)\bigr)\Bigr)(\mathscr{K}_2\bm{z})  +\mathbf{G}_1\bm{h}_1\Bigl(\bm{v}\bigl(\bm{x},\bm{\delta}(t)\bigr)\Bigr)\\ \mathbf{\Sigma}\Bigl(\bm{v}\bigl(\bm{x},\bm{\delta}(t)\bigr)\Bigr) \bigl[\bm{z}(\xi) + (\mathscr{K}_3\bm{z})\bigr] +\bm{h}_2\Bigl(\bm{v}\bigl(\bm{x},\bm{\delta}(t)\bigr)\Bigr)\end{bmatrix}.
\end{align}
Utilizing the abstract representation~\eqref{eq:abstracto}, local well-posedness for the ODE-PDE system~\eqref{eq:originalSystems} may be inferred by showing that $\mathscr{A}$ generates a $C_0$-semigroup on $\mathcal{X}$, and then treating $\bm{f} :\mathcal{X}\times [0,T]\to \mathcal{X}$ as a nonlinear perturbation, according to the classic theory for semilinear evolution equations. Recalling the decomposition~\eqref{eq:Adecomp}, it is actually sufficient to prove that $\mathscr{A}_0$ generates a $C_0$-semigroup on $\mathcal{X}$, since $\mathscr{A}_1\in \mathscr{B}(\mathcal{X}) $ may in turn be regarded as a bounded perturbation. Two preliminary results in this direction, concerning respectively the closedness and quasi-dissipativity of $(\mathscr{A}_0,\mathscr{D}(\mathscr{A}_0))$, are asserted by Propositions~\ref{prop:Closed} and~\ref{prop:quasiDiss} below.

\begin{proposition}[Closedness]\label{prop:Closed}
The operator $(\mathscr{A}_0, \mathscr{D}(\mathscr{A}_0))$ defined as in~\eqref{operator:A_0} is closed.
\begin{proof}
It suffices to show that there exists $\lambda \in \mathbb{R}$ such that $(\mathscr{A}_0-\lambda I_{\mathcal{X}}, \mathscr{D}(\mathscr{A}_0))$ is invertible. Setting
\begin{align}\label{eq:settt}
\begin{split}
\bigl((\mathscr{A}_0-\lambda I_{\mathcal{X}})(\bm{y}_1,\bm{\zeta}_1)\bigr)(\xi) & = \begin{bmatrix} (\mathbf{A}_1-\lambda \mathbf{I}_{2}) \bm{y}_1 + \mathbf{G}_1\mathbf{K}_2\bm{\zeta}_1(1) \\ -\mathbf{\Lambda}\dpd{\bm{\zeta}_1(\xi)}{\xi} + \mathbf{K}_6\bm{\zeta}_1(1) -\lambda \bm{\zeta}_1(\xi) \end{bmatrix} = \begin{bmatrix} \bm{y}_2 \\ \bm{\zeta}_2(\xi) \end{bmatrix}
\end{split}
\end{align}
and solving for the second component yields
\begin{align}\label{eq:solV}
\begin{split}
\bm{\zeta}_1(\xi) & = \mathbf{\Gamma}_0(\xi,\lambda)\mathbf{\Lambda}^{-1}\mathbf{K}_6 \bm{\zeta}_1(1) - \int_0^\xi \mathbf{\Phi}_0\bigl(\xi,\xi^\prime,\lambda\bigr) \mathbf{\Lambda}^{-1}\bm{\zeta}_2(\xi^\prime) \dif \xi^\prime, \quad \xi \in [0,1],
\end{split}
\end{align}
where
\begin{subequations}
\begin{align}
\mathbf{\Gamma}_0(\xi,\lambda) & \triangleq \int_0^\xi \mathbf{\Phi}_0\bigl(\xi,\xi^\prime,\lambda\bigr)\dif \xi^\prime, \\
\mathbf{\Phi}_0(\xi,\tilde{\xi},\lambda) & \triangleq\eu^{-\lambda \mathbf{\Lambda}^{-1}(\xi-\tilde{\xi})}.
\end{align}
\end{subequations}
Accordingly,
\begin{align}
\bigl[\mathbf{I}_2-\mathbf{\Theta}_0(\lambda)\bigr]\bm{\zeta}_1(1) &= - \int_0^1\mathbf{\Phi}_0(1,\xi,\lambda)\mathbf{\Lambda}^{-1}\bm{\zeta}_2(\xi) \dif \xi, 
\end{align}
where
\begin{align}\label{eq:MatrixSigma}
\mathbf{\Theta}_0(\lambda) \triangleq \mathbf{\Gamma}_0(1,\lambda)\mathbf{\Lambda}^{-1}\mathbf{K}_6  = \dfrac{1}{\lambda}\Bigl(\mathbf{I}_{2}- \eu^{-\lambda \mathbf{\Lambda}^{-1}} \Bigr)\mathbf{K}_6.
\end{align}
Since $\mathbf{I}_2-\mathbf{\Theta}_0(\lambda) \in \mathbf{GL}_{2 }(\mathbb{R})$ for sufficiently large $ \lambda \notin \sigma(\mathbf{A}_1)$, combining Eqs.~\eqref{eq:settt} and~\eqref{eq:MatrixSigma} yields
\begin{align}
\bm{y}_1 =  (\mathbf{A}_1-\lambda \mathbf{I}_{2})^{-1}\Biggl[\bm{y}_2 - \mathbf{G}_1\mathbf{K}_2\bigl[\mathbf{I}_2-\mathbf{\Theta}_0(\lambda)\bigr]^{-1}\int_0^1 \mathbf{\Phi}_0(1,\xi,\lambda)\mathbf{\Lambda}^{-1}\bm{\zeta}_2(\xi) \dif \xi\Biggr],
\end{align}
implying that $(\mathscr{A}_0, \mathscr{D}(\mathscr{A}_0))$ is closed.
\end{proof}
\end{proposition}

\begin{proposition}[Quasi-dissipativity]\label{prop:quasiDiss}
The operator $(\mathscr{A}_0, \mathscr{D}(\mathscr{A}_0))$ defined according to~\eqref{operator:A_0} and its adjoint $(\mathscr{A}_0^*, \mathscr{D}(\mathscr{A}_0^*))$ are both quasi-dissipative, that is, they satisfy
\begin{subequations}
\begin{align}
\RE\bigl\langle \mathscr{A}_0(\bm{y},\bm{\zeta}), (\bm{y},\bm{\zeta})\bigr \rangle_{\mathcal{X}} \leq \omega_0 \norm{(\bm{y},\bm{\zeta}(\cdot))}_{\mathcal{X}}^2, && (\bm{y},\bm{\zeta}) \in \mathscr{D}(\mathscr{A}_0), \label{eq:qDiss1}\\
\RE\bigl\langle \mathscr{A}_0^*(\bm{y},\bm{\zeta}), (\bm{y},\bm{\zeta})\bigr \rangle_{\mathcal{X}} \leq \omega_0 \norm{(\bm{y},\bm{\zeta}(\cdot))}_{\mathcal{X}}^2, && (\bm{y},\bm{\zeta}) \in \mathscr{D}(\mathscr{A}_0^*), \label{eq:qDiss2}
\end{align}
\end{subequations}
with constant $\omega_0$ given by
\begin{align}\label{eq:omega0}
\omega_0 & \triangleq \max\left\{\norm{\mathbf{A}_1} + \dfrac{\norm{\mathbf{G}_1\mathbf{K}_2}^2}{\lambda\ped{min}(\mathbf{\Lambda})}, \dfrac{\norm{\mathbf{K}_6}^2}{\lambda\ped{min}(\mathbf{\Lambda})}\right\},
\end{align}
where $\mathbb{R}_{>0} \ni \lambda\ped{min}(\mathbf{\Lambda})$ denotes the smallest eigenvalue of $\mathbf{\Lambda}$.
\begin{proof}
Considering the operator $(\mathscr{A}_0, \mathscr{D}(\mathscr{A}_0))$ and taking the inner product on $\mathcal{X}$ provides
\begin{align}
\begin{split}
 \RE\bigl\langle \mathscr{A}_0(\bm{y},\bm{\zeta}), (\bm{y},\bm{\zeta})\bigr \rangle_{\mathcal{X}} & = \bm{y}^\mathrm{T}\mathbf{A}_1^{\mathrm{T}}\bm{y} + \bm{\zeta}^{\mathrm{T}}(1)(\mathbf{G}_1\mathbf{K}_2)^{\mathrm{T}}\bm{y} - \int_0^1 \dpd{\bm{\zeta}^{\mathrm{T}}(\xi)}{\xi}\mathbf{\Lambda} \bm{\zeta}(\xi) \dif \xi + \bm{\zeta}^{\mathrm{T}}(1)\mathbf{K}_6^{\mathrm{T}}\int_0^1 \bm{\zeta}(\xi) \dif \xi\\
& = \bm{y}^\mathrm{T}\mathbf{A}_1^{\mathrm{T}}\bm{y} + \bm{\zeta}^{\mathrm{T}}(1)(\mathbf{G}_1\mathbf{K}_2)^{\mathrm{T}}\bm{y}  - \dfrac{1}{2}\bm{\zeta}^{\mathrm{T}}(1)\mathbf{\Lambda} \bm{\zeta}(1) + \bm{\zeta}^{\mathrm{T}}(1)\mathbf{K}_6^{\mathrm{T}}\int_0^1 \bm{\zeta}(\xi) \dif \xi,
\end{split}
\end{align}
for all $(\bm{y},\bm{\zeta}) \in \mathscr{D}(\mathscr{A}_0)$. Applying Cauchy-Schwarz and then the generalized Young's inequality for products to the second term on the right-hand side yields
\begin{align}
\begin{split}
 \RE\bigl\langle \mathscr{A}_0(\bm{y},\bm{\zeta}), (\bm{y},\bm{\zeta})\bigr \rangle_{\mathcal{X}} &  \leq \biggl(\norm{\mathbf{A}_1}+\dfrac{\varepsilon}{2}\norm{\mathbf{G}_1\mathbf{K}_2}^2 \biggr)\norm{\bm{y}}_2^2+ \dfrac{\varepsilon}{2}\norm{\mathbf{K}_6}^2\norm{\bm{\zeta}(\cdot,t)}_{L^2((0,1);\mathbb{R}^{2})}^2 \\
& \quad  -\biggl(\dfrac{\lambda\ped{min}(\mathbf{\Lambda})}{2} -\dfrac{1}{\varepsilon} \biggr)\norm{\bm{\zeta}(1)}_{2}^2,
\end{split}
\end{align}
for some $\varepsilon \in \mathbb{R}_{>0}$.
In particular, selecting $\varepsilon \triangleq 2/\lambda\ped{min}(\mathbf{\Lambda})$ gives~\eqref{eq:qDiss1}.
Moreover, the adjoint operator $(\mathscr{A}_0^*, \mathscr{D}(\mathscr{A}_0^*))$ of $(\mathscr{A}_0, \mathscr{D}(\mathscr{A}_0))$ may be deduced to have the form
\begin{subequations}\label{operator:A_0*}
\begin{align}
& \bigl(\mathscr{A}_0^*(\bm{y},\bm{\zeta})\bigr)(\xi)  \triangleq \begin{bmatrix} \mathbf{A}_1^\mathrm{T}\bm{y} \\  \mathbf{\Lambda}\dpd{\bm{\zeta}(\xi)}{\xi}\end{bmatrix}, \\
& \mathscr{D}(\mathscr{A}_0^*) \triangleq \Biggl\{(\bm{y},\bm{\zeta}) \in \mathcal{Y} \mathrel{\Bigg|} \mathbf{\Lambda} \bm{\zeta}(1) = (\mathbf{G}_1\mathbf{K}_2)^{\mathrm{T}}\bm{y} + \mathbf{K}_6^{\mathrm{T}}\int_0^1 \bm{\zeta}(\xi) \dif \xi \Biggr\}.
\end{align}
\end{subequations}
Hence, similar manipulations as previously provide
\begin{align}
\begin{split}
& \RE\bigl\langle \mathscr{A}_0^*(\bm{y},\bm{\zeta}), (\bm{y},\bm{\zeta})\bigr \rangle_{\mathcal{X}} = \bm{y}^{\mathrm{T}}\mathbf{A}_1\bm{y} + \dfrac{1}{2}\bm{\zeta}^{\mathrm{T}}(1)\mathbf{\Lambda} \bm{\zeta}(1)  - \dfrac{1}{2}\bm{\zeta}^{\mathrm{T}}(0)\mathbf{\Lambda}\bm{\zeta}(0) + \dfrac{1}{2}\int_0^1\bm{\zeta}^{\mathrm{T}}(\xi)\mathbf{\Lambda} \bm{\zeta}(\xi) \dif \xi,
\end{split}
\end{align}
which, for $(\bm{y},\bm{\zeta}) \in \mathscr{D}(\mathscr{A}_0^*)$, yields~\eqref{eq:qDiss2}.
\end{proof}
\end{proposition}

The proofs of Theorems~\ref{thm:mild} and~\ref{thm:class} are given below.
\begin{proof}[Proof of Theorem~\ref{thm:mild}]
Since $\mathbb{R}^{2} \times C_0^1([0,1];\mathbb{R}^{2}) \subset \mathscr{D}(\mathscr{A}_0)$, the operator $(\mathscr{A}_0,\mathscr{D}(\mathscr{A}_0))$ as defined in~\eqref{operator:A_0} is dense, i.e., $\overline{\mathscr{D}(\mathscr{A}_0)} = \mathcal{X}$. Moreover, it is closed and quasi-dissipative according to Propositions~\ref{prop:Closed} and~\ref{prop:quasiDiss}. It follows from Lumer-Phillips' Theorem that, for $t \in [0,T]$, $\mathscr{A}_0 \in \mathscr{G}(\mathcal{X};1,\omega_0)$, with $\omega_0$ as in~\eqref{eq:omega0} (\cite{Zwart}, Corollary 2.3.3). Consider now the decomposition~\eqref{eq:Adecomp}, with $\mathscr{D}(\mathscr{A}) \equiv \mathscr{D}(\mathscr{A}_0)$. Since $\mathscr{A}_0 \in \mathscr{G}(\mathcal{X};1,\omega_0)$ and $ \mathscr{A}_1\in \mathscr{B}(\mathcal{X}) $, there exists $M \in \mathbb{R}_{\geq 0}$ such that $\mathscr{A}\in \mathscr{G}(\mathcal{X};1, \omega)$ with $\omega \triangleq \omega_0 + M$. 
Moreover, $\bm{\delta} \in C^0([0,T];\mathbb{R}^{2})$ implies $\bm{v} \in C^0(\mathbb{R}^{2}\times[0,T];\mathbb{R}^{2 })$. Since $\mathbf{\Sigma} \in C^0(\mathbb{R}^{2 };\mathbf{M}_{2 }(\mathbb{R}))$, $\bm{h}_1 \in C^0(\mathbb{R}^{2 };\mathbb{R}^{2})$, and $\bm{h}_2 \in C^0(\mathbb{R}^{2 };\mathbb{R}^{2 })$ are locally Lipschitz continuous by assumption, $\bm{f} : \mathcal{X} \times [0,T] \to \mathcal{X}$ is continuous in $t$ on $[0,T]$, and locally Lipschitz on $\mathcal{X}$, uniformly in $t$ on bounded intervals. Hence, all the Hypotheses of Theorem 6.1.4 in \cite{Pazy} are verified, ensuring the existence of a maximal time $t\ped{max} \leq \infty$, and a unique mild solution $(\bm{x},\bm{z}) \in C^0([0,t\ped{max});\mathcal{X})$ for all ICs $(\bm{x}_0,\bm{z}_0)\in \mathcal{X}$.
\end{proof}

\begin{proof}[Proof of Theorem~\ref{thm:class}]
$\bm{\delta} \in C^1([0,T];\mathbb{R}^{2 })$ implies $\bm{v} \in C^1(\mathbb{R}^{2}\times[0,T];\mathbb{R}^{2 })$. Since $\mathbf{\Sigma} \in C^1(\mathbb{R}^{2 };\mathbf{M}_{2 }(\mathbb{R}))$, $\bm{h}_1 \in C^1(\mathbb{R}^{2 };\mathbb{R}^{2})$, and $\bm{h}_2 \in C^1(\mathbb{R}^{2 };\mathbb{R}^{2 })$ by assumption, $\bm{f} : \mathcal{X} \times [0,T] \to \mathcal{X}$ is continuously differentiable from $\mathcal{X} \times [0,T]$ into $\mathcal{X}$. Hence, all the hypotheses of Theorem 6.1.5 in \cite{Pazy} are verified, ensuring the existence of a maximal time $t\ped{max} \leq \infty$, and a unique classical solution $(\bm{x},\bm{z}) \in C^1([0,t\ped{max});\mathcal{X})\cap C^0([0,t\ped{max});\mathscr{D}(\mathscr{A}))$ for all ICs $(\bm{x}_0,\bm{z}_0) \in \mathscr{D}(\mathscr{A})$, which is equivalent to $(\bm{x},\bm{z}) \in C^1([0,t\ped{max});\mathcal{X})\cap C^0([0,t\ped{max});\mathcal{Y})$ satisfying the BC~\eqref{eq:BCoriginal}.
\end{proof}

\subsection{Proof of Theorem~\ref{thm:global}}\label{app:thmGlobal}
The proof of Theorem~\ref{thm:global} is given below.
\begin{proof}[Proof of Theorem~\ref{thm:global}]
The following Lyapunov function candidate is considered:
\begin{align}\label{eq:Lyapunov}
V\bigl(\bm{x}(t),\bm{z}(\cdot,t)\bigr) \triangleq \dfrac{1}{2}\bm{x}^{\mathrm{T}}(t)\bm{x}(t) + \dfrac{1}{2}\int_0^1 \bm{z}^{\mathrm{T}}(\xi,t)\mathbf{P}(\xi)\bm{z}^{\mathrm{T}}(\xi,t) \dif \xi, 
\end{align}
where $C^0([0,1];\mathbf{Sym}_{2}(\mathbb{R})) \ni \mathbf{P} \triangleq \diag\{P_1, P_2\}$, with $\mathbf{P}(\xi) \succ \mathbf{0}$, satisfies Eq.~\eqref{eq:ineqP} whenever~\ref{th:ext.1} holds. For every $(\bm{x}_0,\bm{z}_0) \in \mathscr{D}(\mathscr{A})$, the above Lyapunov function is differentiable in a classical sense.
Hence, taking the derivative of Eq.~\eqref{eq:Lyapunov} along the dynamics~\eqref{eq:originalSystems} and integrating by parts provides
\begin{align}
\begin{split}
\dot{V}(t) & \leq \norm{\mathbf{A}_1}\norm{\bm{x}(t)}_2^2 + \bm{x}^{\mathrm{T}}(t)\mathbf{G}_1\biggl[(\mathscr{K}_1\bm{z})(t) + \mathbf{\Sigma}\Bigl(\bm{v}\bigl(\bm{x}(t),\bm{\delta}(t)\bigr)\Bigr)(\mathscr{K}_2\bm{z})(t)+\bm{h}_1\Bigl(\bm{v}\bigl(\bm{x}(t),\bm{\delta}(t)\bigr)\Bigr)\biggr] \\
& \quad - \dfrac{\lambda\ped{min}\bigl(\mathbf{P}(1)\mathbf{\Lambda}\bigr)}{2}\norm{\bm{z}(1,t)}_2^2 + \dfrac{1}{2}\int_0^1\bm{z}^{\mathrm{T}}(\xi,t)\dod{\mathbf{P}(\xi)}{\xi}\mathbf{\Lambda}\bm{z}(\xi,t) \dif \xi \\
& \quad + \int_0^1 \bm{z}^{\mathrm{T}}(\xi,t)\mathbf{P}(\xi)\mathbf{\Sigma}\Bigl(\bm{v}\bigl(\bm{x}(t),\bm{\delta}(t)\bigr)\Bigr)\bigl[\bm{z}(\xi,t) + (\mathscr{K}_3\bm{z})(t)\bigr] \dif \xi + \int_0^1 \bm{z}^{\mathrm{T}}(\xi,t)\mathbf{P}(\xi)\dif \xi (\mathscr{K}_4\bm{z})(t) \dif \xi\\
& \quad +   \int_0^1 \bm{z}^{\mathrm{T}}(\xi,t)\mathbf{P}(\xi)\dif \xi\bm{h}_2 \Bigl(\bm{v}\bigl(\bm{x}(t),\bm{\delta}(t)\bigr)\Bigr)\dif\xi, \quad t \in (0,T),
\end{split}
\end{align}
where $\mathbb{R}_{>0}\ni\lambda\ped{min}\bigl(\mathbf{P}(1)\mathbf{\Lambda}\bigr)$ denotes the smallest eigenvalue of $\mathbf{P}(1)\mathbf{\Lambda}$.

Equations~\eqref{eq:hCond} in conjunction with~\eqref{eq:relVel} imply that there exist $\gamma_1,\gamma_2 \in \mathbb{R}_{\geq 0}$ such that
\begin{subequations}
\begin{align}
\bm{x}^{\mathrm{T}}(t)\mathbf{G}_1\bm{h}_1\Bigl(\bm{v}\bigl(\bm{x}(t),\bm{\delta}(t)\bigr)\Bigr) & \leq \gamma_1\Bigl(\norm{\bm{x}(t)}_2^2 + \norm{\bm{\delta}(t)}_2^2 +b_1^2\Bigr), \\
\int_0^1 \bm{z}^{\mathrm{T}}(\xi,t)\mathbf{P}(\xi)\dif \xi\bm{h}_2 \Bigl(\bm{v}\bigl(\bm{x}(t),\bm{\delta}(t)\bigr)\Bigr)\dif \xi& \leq \gamma_2\Bigl(\norm{(\bm{x}(t),\bm{z}(\cdot,t))}_{\mathcal{X}}^2 + \norm{\bm{\delta}(t)}_2^2+b_2^2 \Bigr).
\end{align}
\end{subequations}
Furthermore, from Eqs.~\eqref{eq:operatorK_2} and~\eqref{eq:operatorK_3}, and Hypotheses~\ref{th:ext.1} and~\ref{th:ext.2}, it follows the existence of $\gamma_3,\gamma_4 \in \mathbb{R}_{\geq 0}$ such that
\begin{subequations}
\begin{align}
\bm{x}^{\mathrm{T}}(t) \mathbf{G}_1\mathbf{\Sigma}\Bigl(\bm{v}\bigl(\bm{x}(t),\bm{\delta}(t)\bigr)\Bigr)(\mathscr{K}_2\bm{z})(t) & \leq \gamma_3 \norm{(\bm{x}(t),\bm{z}(\cdot,t))}_{\mathcal{X}}^2 , \\
\int_0^1 \bm{z}^{\mathrm{T}}(\xi,t)\mathbf{P}(\xi)\mathbf{\Sigma}\Bigl(\bm{v}\bigl(\bm{x}(t),\bm{\delta}(t)\bigr)\Bigr)\bigl[\bm{z}(\xi,t) + (\mathscr{K}_3\bm{z})(t)\bigr] \dif \xi & \leq \gamma_4 \norm{\bm{z}(\cdot,t)}_{L^2((0,1);\mathbb{R}^{2})}^2.
\end{align}
\end{subequations}
Moreover, recalling Eqs.~\eqref{eq:operatorK_1} and~\eqref{eq:operatorK_4}, and applying Cauchy-Schwarz and then the generalized Young's inequality for products yields
\begin{subequations}
\begin{align}
\bm{x}^{\mathrm{T}}(t)\mathbf{G}_1(\mathscr{K}_1\bm{z})(t) & \leq \gamma_5\norm{(\bm{x}(t),\bm{z}(\cdot,t))}_{\mathcal{X}}^2 + \dfrac{1}{2\varepsilon}\norm{\bm{z}(1,t)}_2^2, \\
\int_0^1 \bm{z}^{\mathrm{T}}(\xi,t)\mathbf{P}(\xi)\dif \xi (\mathscr{K}_4\bm{z})(t) \dif \xi& \leq \gamma_6\norm{\bm{z}(\cdot,t)}_{L^2((0,1);\mathbb{R}^{2})}^2 + \dfrac{1}{2\varepsilon}\norm{\bm{z}(1,t)}_2^2,
\end{align}
\end{subequations}
for some $\gamma_5,\gamma_6 \in \mathbb{R}_{\geq 0}$, and some $\varepsilon \in \mathbb{R}_{>0}$ to be appropriately selected. In particular, specifying $\varepsilon = 2/\lambda\ped{min}(\mathbf{P}(1)\mathbf{\Lambda})$ ensures the existence of $\gamma,\beta \in \mathbb{R}_{>0}$ such that
\begin{align}
\dot{V}(t) & \leq \gamma V(t) + \beta \Bigl(\norm{\bm{\delta}(t)}_2^2 + \norm{\bm{b}}_2^2\Bigr), \quad t \in (0,T),
\end{align}
where $\mathbb{R}^2 \ni \bm{b} \triangleq [b_1\; b_2]^{\mathrm{T}}$.
Thus, applying Grönwall-Bellman's inequality gives
\begin{align}
\begin{split}
V(t) & \leq \eu^{\gamma t}V(0) +\beta \int_0^t \eu^{\gamma(t-t^\prime)}\Bigl(\norm{\bm{\delta}(t^\prime)}_2^2 + \norm{\bm{b}}_2^2\Bigr) \dif t^\prime \\
&  \leq \eu^{\gamma t}V(0) + \beta\dfrac{\eu^{\gamma t}-1}{\gamma}\Bigl(\norm{\bm{\delta}(\cdot)}_\infty^2+\norm{\bm{b}}_2^2\Bigr), \quad t\in [0,T].
\end{split}
\end{align}
By observing that the Lyapunov function $V(\bm{x}(t),\bm{z}(\cdot,t))$ in Eq.~\eqref{eq:Lyapunov} is equivalent to the squared norm $\norm{(\bm{x}(t), \bm{z}(\cdot,t))}_{\mathcal{X}}^2$ on $\mathcal{X}$, the claim follows for all ICs $(\bm{x}_0,\bm{z}_0) \in \mathscr{D}(\mathscr{A})$ under the assumptions as Theorem~\ref{thm:class}. The generalisation to ICs $(\bm{x}_0,\bm{z}_0) \in \mathcal{X}$ is a consequence of the fact that $\overline{\mathscr{D}(\mathscr{A})} = \mathcal{X}$, and the mapping $(\bm{x}_0,\bm{z}_0) \to (\bm{x}(t), \bm{z}(\cdot,t))$ is Lipschitz continuous from $\mathcal{X}$ into $C^0([0,T];\mathcal{X})$ (Theorem 6.1.2 in \cite{Pazy}). 
\end{proof}

\subsection{Proof of Theorem~\ref{thm:Ulin}}\label{app:ProofLin}
The proof of Theorem~\ref{thm:Ulin} is given below.
\begin{proof}[Proof of Theorem~\ref{thm:Ulin}]
For every $(\bm{x}^\star, \bm{z}^\star, \bm{\delta}^\star) \in \mathcal{Y}\times \mathbb{R}^2$, the ODE-PDE system may be cast in abstract form as
\begin{subequations}
\begin{align}
\dod{}{t}\begin{bmatrix} \tilde{\bm{x}}(t) \\ \tilde{\bm{z}}(t) \end{bmatrix} & = \tilde{\mathscr{A}}\begin{bmatrix} \tilde{\bm{x}}(t) \\ \tilde{\bm{z}}(t) \end{bmatrix} + \tilde{\bm{f}}(t), \quad t\in(0,T), \\
\begin{bmatrix} \tilde{\bm{x}}(0) \\ \tilde{\bm{z}}(0) \end{bmatrix} & = \begin{bmatrix}\tilde{\bm{x}}_0 \\ \tilde{\bm{z}}_0 \end{bmatrix},
\end{align}
\end{subequations}
where $(\tilde{\mathscr{A}}, \mathscr{D}(\tilde{\mathscr{A}}))$, with $\tilde{\mathscr{A}} = \mathscr{A} + \mathscr{A}_2$, $\mathscr{D}(\tilde{\mathscr{A}}) = \mathscr{D}(\mathscr{A})$, and $\mathscr{A}_2 \in \mathscr{B}(\mathcal{X})$ given by
\begin{align}
\bigl(\mathscr{A}_2(\bm{y},\bm{\zeta})\bigr)(\xi) \triangleq \begin{bmatrix}  \mathbf{G}_1\mathbf{H}_1\bigl(\bm{x}^\star,\bm{\delta}^\star,\bm{z}^\star\bigr)\mathbf{A}_2\bm{y} +\mathbf{G}_1\Bigl[(\mathscr{K}_1\bm{\zeta}) + \mathbf{\Sigma}\bigl(\bm{v}(\bm{x}^\star,\bm{\delta}^\star)\bigr)(\mathscr{K}_2\bm{\zeta})\Bigr] \\   \mathbf{H}_2\bigl(\bm{x}^\star,\bm{\delta}^\star,\bm{z}^\star(\xi)\bigr)\mathbf{A}_2\bm{y} + \mathbf{\Sigma}\bigl(\bm{v}(\bm{x}^\star,\bm{\delta}^\star)\bigr)\bigl[\bm{\zeta}(\xi) + (\mathscr{K}_3\bm{\zeta})\bigr] \end{bmatrix},
\end{align}
and $\tilde{\bm{f}} : [0,T]\to \mathcal{X}$ reading
\begin{align}
\tilde{\bm{f}}(t) & \triangleq \begin{bmatrix} \mathbf{G}_1\tilde{\mathbf{B}}_1\bigl(\bm{x}^\star,\bm{\delta}^\star,\bm{z}^\star\bigr) \\ \tilde{\mathbf{B}}_2\bigl(\bm{x}^\star,\bm{\delta}^\star,\bm{z}^\star(\xi)\bigr) \end{bmatrix}\tilde{\bm{\delta}}(t).
\end{align}
Since, from Theorem~\ref{thm:mild}, $\mathscr{A} \in \mathscr{G}(\mathcal{X};1,\omega)$ and $ \mathscr{A}_2\in \mathscr{B}(\mathcal{X}) $, there exists $\tilde{M} \in \mathbb{R}_{\geq 0}$ such that $\tilde{\mathscr{A}}\in \mathscr{G}(\mathcal{X};1, \tilde{\omega})$ with $\tilde{\omega} \triangleq \omega + \tilde{M}$. This ensures the existence of a unique mild solution $(\tilde{\bm{x}},\tilde{\bm{z}})\in C^0([0,T];\mathcal{X})$ for all ICs $(\tilde{\bm{x}}_0,\tilde{\bm{z}}_0) \in \mathcal{X}$ and inputs $\tilde{\bm{\delta}} \in L^p((0,T);\mathbb{R}^2)$, $p\geq 1$, as defined in \cite{Pazy} (Definition 4.2.3); from Corollary 4.2.5 in \cite{Pazy}, it also follows that, for all ICs $(\tilde{\bm{x}}_0,\tilde{\bm{z}}_0) \in \mathscr{D}(\tilde{\mathscr{A}})$ and inputs $\tilde{\bm{\delta}} \in C^1([0,T];\mathbb{R}^2)$, the solution is classical, that is, $(\tilde{\bm{x}},\tilde{\bm{z}}) \in C^1([0,T];\mathcal{X})\cap C^0([0,T];\mathscr{D}(\tilde{\mathscr{A}}))$, which is equivalent to $(\tilde{\bm{x}},\tilde{\bm{z}}) \in C^1([0,T];\mathcal{X})\cap C^0([0,T];\mathcal{Y})$ satisfying the BC~\eqref{eq:BCoriginalLin}.
\end{proof}

\end{document}